\documentclass[nohyperref]{article}

\usepackage{microtype}
\usepackage{graphicx}
\usepackage{subfigure}
\usepackage{booktabs} 

\usepackage{hyperref}

\usepackage[accepted]{icml2022}

\usepackage{amsmath}
\usepackage{amssymb}
\usepackage{mathtools}
\usepackage{amsthm}

\usepackage[capitalize,noabbrev]{cleveref}
\usepackage{colortbl}
\theoremstyle{plain}
\newtheorem{theorem}{Theorem}[section]

\theoremstyle{definition}

\theoremstyle{remark}

\def\bTh{{\bf \Theta}}
\def\bPi{{\bf \Pi}}

\def\bnabla{\boldsymbol{\nabla}}

\usepackage[textsize=tiny]{todonotes}

\definecolor{color_table}{rgb}{1.0, 0.44, 0.37}
\icmltitlerunning{BI for AI:  Energy Conserving Descent}

\begin{document}

\twocolumn[
\icmltitle{Born-Infeld (BI) for AI: 
           Energy-Conserving Descent (ECD) for Optimization}



\icmlsetsymbol{equal}{*}

\begin{icmlauthorlist}
\icmlauthor{G. Bruno De Luca}{equal,yyy}
\icmlauthor{Eva Silverstein}{equal,yyy}
\end{icmlauthorlist}

\icmlaffiliation{yyy}{Stanford Institute for Theoretical Physics, Stanford University, Stanford, CA, 94306, USA}

\icmlcorrespondingauthor{G. Bruno De Luca}{gbdeluca@stanford.edu}
\icmlcorrespondingauthor{Eva Silverstein}{evas@stanford.edu}

\icmlkeywords{Machine Learning, ICML}

\vskip 0.3in
]




\printAffiliationsAndNotice{\icmlEqualContribution} 

\begin{abstract}
We introduce a novel framework for optimization based on energy-conserving Hamiltonian dynamics in a strongly mixing (chaotic) regime and establish its key properties analytically and numerically.  The prototype is a discretization of Born-Infeld dynamics, with a squared relativistic speed limit depending on the objective function. This class of frictionless, energy-conserving optimizers proceeds unobstructed until slowing naturally near the minimal loss, which dominates the phase space volume of the system.  Building from studies of chaotic systems such as dynamical billiards, we formulate a specific algorithm with good performance on machine learning and PDE-solving tasks, including generalization. It cannot stop at a high local minimum, 
an advantage in non-convex loss functions, and proceeds faster than GD+momentum in shallow valleys.

\end{abstract}

\section{Introduction and Summary}\label{sec:intro}
Many scientific and technological problems, including machine learning (ML), require optimization, the process of evolving to or near to the minimum of a nontrivial loss function $F(\bTh)$ depending on a high dimensional space of parameters $\bTh$.   The loss function may be highly non-convex and/or contain shallow long valleys.  It traditionally works -- quite well in many cases -- via a form of gradient descent (GD), with many improvements derived over the years involving momentum and adaptive features\footnote{See e.g. \cite{bottou2018optimization} for a review and \cite{kingma2017adam, reddi2019convergence} for examples of adaptive optimization.}.  A standard example is (stochastic) GD with momentum \cite{polyak}  which we will denote (S)GDM.  In physics terms, these standard algorithms involve a noisy and discretized form of frictional evolution of a particle on a complicated potential energy landscape, without conservation of  total kinetic and potential energy $E$. More precisely, given an objective function $F(\bTh)$, one can define a \emph{potential energy}:
\begin{equation}\label{eq:V-F-shift-def}
    V(\bTh) \equiv F(\bTh) - \Delta V\:,
\end{equation}
where $\Delta V$ is a (possibly zero) constant shift  whose adaptive tuning we will address below. Through this map, optimization  can be formulated as (discretized) physical evolution from an initial point $\bTh_0$,  with the goal of reaching sufficiently low values of $V$. Standard methods correspond to friction-dominated evolution, relying on some form of dissipation of energy  to achieve this. 

In this work, we show that friction is not necessary: energy-conserving dynamics (ECD) provides a distinctive class of optimization algorithms with some favorable properties, including increased calculability  of its behavior, resulting from the $E$ conservation summarized in Table \ref{table:comparison}. The minimal version contains zero friction, and conserves energy, yet essentially stops moving near vanishing loss.  This may sound paradoxical on first glance, but it is straightforward.  One example which we focus on in this work is based on relativistic Born-Infeld (BI) dynamics, with the (squared) speed limit\footnote{Relativistic dynamics for optimization was considered previously in the interesting work \cite{RGD2020}, but there the speed of light was taken to be a constant and friction was introduced to enable the evolution to converge; see also related work in \cite{lu2016relativistic}.  Our approach was motivated by the mechanism \cite{Alishahiha:2004eh} for early-universe inflationary cosmology, a subject which also involves evolution on a potential landscape.} $c_{rel}$ depending on $V(\bTh)$ as: 
\begin{equation}\label{eq:SL-V-F-shift-def}
    c_{rel}^2 = V(\bTh) .
\end{equation}
As a consequence of this, the evolution essentially stops when $V(\bTh)\to 0.$
 As defined in \eqref{eq:V-F-shift-def},
$\Delta V$ represents a constant hyper-parameter sometimes needed to shift the objective to $V=0$; although adaptively tunable, we note that $\Delta V=0$ in practice for a number of standard problems including ML tasks with achievable vanishing loss objectives
and partial differential equation (PDE) solving with $F$ given by the summed squares of the equations.  When $V$ vanishes, so does the speed of motion in parameter space.  Another example is to take $mass=1/V$ for ordinary (non-relativistic) particle motion; again even without friction slowing it down, the system ultimately slows as $V\to 0, \text{mass}\to\infty$.

In various applications, literally achieving the globally minimal loss may lead to overfitting and failure of generalization.  For BI, we find that the speed limit kicks in soon enough to avoid this in  small benchmark experiments and our own synthetic PDE solving tests.   

Our {\bf contributions} are to introduce ECD, the frictionless energy-conserving class of optimizers with {\it mixing} dynamics (introduced in \S\ref{sub:mixing}), realized concretely with {\it BBI} (Algorithm \ref{alg:BBI}), derive its key properties and advantages (Table \ref{table:comparison}) and confirm them experimentally in synthetic loss functions, PDEs, MNIST and CIFAR.
\begin{table}[t]
\caption{ECD versus frictional optimizers.  Energy conservation improves calculational control  of optimization and generalization.
}
\label{table:comparison}
\vskip 0.15in
\begin{center}
\begin{small}
\begin{sc}
\begin{tabular}{ll}
\toprule
ECD & Friction ((S)GDM, \dots)   \\
\midrule
\rowcolor{color_table!15}Conserves energy E    & Friction drains E   \\
\rowcolor{color_table!5}Cannot get stuck  &  Can stop in high \\
\rowcolor{color_table!5} ~in high local minimum & ~local minimum \\

\rowcolor{color_table!15}Cannot overshoot & Can overshoot  \\
\rowcolor{color_table!15}~ $V=0=\bnabla V$ & $~V=0=\bnabla V$\\
\rowcolor{color_table!5}Depends on $V$ and $\nabla V$     & depends only on $\bnabla V$ \\
\rowcolor{color_table!15}On shallow region:    & On shallow region:          \\
\rowcolor{color_table!15}$~\theta\sim e^{-mt/\sqrt{2}}~~$\eqref{eq:bi-shallow} & $~\theta \sim e^{-m^2 t/f}~~$ \eqref{eq:gd-nomom-shallow}   \\
\rowcolor{color_table!5} Analytic prediction      &  Stochastic intuition  \\
\rowcolor{color_table!5}$~~$  for distribution      & $~~$  for distribution         \\
\bottomrule
\end{tabular}
\end{sc}
\end{small}
\end{center}
\vskip -0.1in
\end{table}
\section{Speed-limited Energy Conserving Frictionless Hamiltonian Dynamics}\label{sec:continuum-dynamics}

In physical language, GDM is a discretized version of particle motion with friction on a potential energy function $ V(\bTh)$.  The motion in the potential is governed by Newton's laws of motion, force = mass $\times$ acceleration, $-\bnabla V-f\dot\bTh = m\ddot\bTh  $ with coefficient of friction $f$.  
This is equivalent to the first-order form $\mathbf{p} = m\dot\bTh ,  \dot{\mathbf{p}} = -\mathbf{p} f /m-\bnabla V $ in terms of the position $\bTh$ and momentum $\mathbf{p}$, whose appropriate discretization leads to GDM algorithm. Friction leads to violation of energy conservation for the particle (physically, transfer of energy to another sector). 

In relativistic particle mechanics, the evolution equations are different, with bounded
speed.
For squared speed limit $V$ one finds frictionless,  energy-conserving continuum evolution equations given in terms of positions $\theta_i$ (components of $\bTh$) and momenta $\pi_i\equiv \frac{\dot\theta_i}{\sqrt{1-\frac{\dot{\bTh}^2}{V}}}$ (components of the momentum vector $\bPi$):

\begin{equation}\label{eom}
\dot\theta_i = \pi_i \frac{V(\bTh)}{E} , ~~
\dot\pi_i =-\frac{\partial_i V}{2}\left(\frac{E}{V}+\frac{V}{E}\right)
\end{equation}
written in terms of the conserved energy $E$
\begin{equation}\label{eq:Econs}
E=\frac{V}{\sqrt{1-\frac{\dot{\bTh}^2}{V}}}=\sqrt{V(V+\bPi^2)} = \text{constant}. ~~~
\end{equation}
These equations are concisely derived from the appropriate physical {\it action principle}\footnote{For a succinct introduction to this general formalism and the role of energy conservation classical physics, see e.g. lectures 6-8 of \cite{Susskindminimum}  or \cite{landau1976mechanics} Chapters 1,2,7.}  in Appendix \ref{app:derive-eoms-S-H} starting from the action $S=-\int dt V(\bTh)\sqrt{1-\frac{\dot{\bTh}^2}{V(\bTh)}}$ \cite{BornInfeld, Silverstein:2003hf}. $E$ conservation ($dE/dt=0$)  in the continuum theory follows from applying \eqref{eom} to a time derivative of the third expression in \eqref{eq:Econs}.  

This implies a first-order discrete update rule 
\begin{equation}\label{eomdiscrete2}
\pi_i(t+\Delta t)-\pi_i(t) = -\Delta t\frac{\partial_i V(\bTh(t))}{{2}}\left(\frac{E}{V}+\frac{V}{E}\right)
\end{equation}

\begin{equation}\label{eomdiscrete1}
\theta_i(t+\Delta t)-\theta_i(t) = \Delta t~\pi_i(t{+\Delta t}) \frac{V(\bTh(t))}{E}
\end{equation}
As we will see more in detail below, noise (minibatches) correspond to a time-dependent sequence of jumps in $V(\bTh, t)$. 

This will serve as a starting point for our algorithm, although we will be led to include more features beyond the minimal rule \eqref{eomdiscrete2}-\eqref{eomdiscrete1}.  
Specifically, we will rescale $\bPi$ to restore energy conservation  once violated by the discrete updates -- even in the noisy case of minibatches -- and we will include options for randomized bounces at fixed $E$ to avoid stable or quasistable orbits in the motion.     

Energy conservation immediately implies 
\begin{theorem}
\label{thm:not-stuck}
If $V\ne 0$ and $V\ne E$ then $\dot\bTh\ne 0$ in the continuum evolution, and $\bTh(t+\Delta t)\ne \bTh(t)$ for the discrete algorithm. 
\end{theorem}
\begin{proof} 
From \eqref{eq:Econs} if $\dot\bTh=0$ and $V\ne 0$ then $E=V$ in the continuum. 
This extends to the discrete case since the algorithm restores $E$ by rescaling ${\bf \Pi}$ if numerical errors in energy conservation build up as a result of the discreteness. Specifically, 
the update rule (5), together with the fact that $E$ is constant, makes it impossible to obtain ${\bf\Theta}_{t+1} = {\bf \Theta}_{t}$ (stalling on parameter space) 
unless $V=0$ or $V=E, {\bf \Pi}=0$.  (The latter only persists to later steps (4) if $\partial V=0$ with $V=E$ at initialization, a set of measure 0 easily avoided with initial energy $\delta E$.)        
\end{proof}
Therefore, initialization with $E\ge V\gg V_{\text{objective}}$ ensures the impossibility of getting stuck in a high local minimum with $V\gg V_\text{objective}$; illustrated here for BI, this extends to ECD in general.  

In various problems, shallow directions in the loss landscape are important \cite{li2018measuring, Fort_Scherlis_2019, draxler2019essentially}.  Here also there is a contrast between ECD (e.g. BI) and GDM as well as their stochastic versions.  To get a sense for this, consider a simple shallow direction $V=\frac{1}{2}m^2\theta^2$ (with $m^2$ the smallest eigenvalue of the Hessian).  Gradient descent alone, or friction dominated motion generally, follows the solution
\begin{equation}\label{eq:gd-nomom-shallow}
    \dot{\bTh} \simeq -\frac{\bnabla V}{f}\to \frac{m^2\bTh}{f}, ~~~ \bTh \sim e^{-m^2 t/f}
\end{equation}
with $f$ the friction coefficient.  BI on the other hand satisfies the speed limit 
\begin{equation}\label{eq:bi-shallow}
    |\dot{\bTh}|\le \sqrt{V} \to m \theta/\sqrt{2}, ~~~ \bTh\sim e^{-mt/\sqrt{2}} 
\end{equation}
For a fixed value of the friction $f$, as we reduce $m$ so that the potential becomes more shallow, the GD speed decreases quadratically while the BI speed decreases only linearly, thus moving faster toward the objective.  We find this distinction survives in the discrete hyper-parameter optimized experiments detailed below.      

A related distinction has to do with overshooting $V=0$.  In the continuum, it is impossible for BI to do this since it would violate the speed limit, whereas GDM can overshoot and oscillate (or leave the $V=0$ basin entirely).   The discretized version can overshoot $V=0$ by a small amount, at which point the algorithm optionally either stops the evolution or adaptively adjusts $\Delta V$ if desired (see below).  

If initialized at zero velocity ($\dot\bTh=0$), relativistic effects are negligible to begin with:  in that case the Born-Infeld dynamics starts out like a non-relativistic particle exploring the $V$ landscape without friction. Even without friction, the speed limit ensures that it slows dramatically near $V=0$.

\section{Mixing and Chaos:  ECD with Dispersing Elements}\label{sec:chaos}

To avoid ECD (e.g. BI) getting into long-lived orbits high on the loss landscape $V$, we exploit the analytic power of $E$ conservation and import results from the field of Hamiltonian dynamical systems.  This also enables predictions for the late-time distribution of ECD trajectories.  

\subsection{Mixing and Distributions}\label{sub:mixing}
The \emph{phase space} of a physical system is the space of accessible momenta $\bPi$ and positions $\bTh$.   A chaotic system exhibits sensitive dependence on initial conditions, enabling initially closeby trajectories to explore very different regions of the phase space; in its strongest forms  most periodic orbits are unstable. 
For exponentially mixing systems, 
after a finite timescale $t_\text{mix}$, the probability to find a particle -- sampled from those in a small initial droplet -- in a particular region $R$ of the phase space is given to good approximation by the volume (`measure') of $R$.
\footnote{A canonical example is the mixing of milk in coffee.}
This volume is parametrically large as $V\to 0$ for BI.  The total phase space volume for an $n$-dimensional parameter space is given by
\begin{equation}\label{psvol}
\text{Vol}({\cal M}) = \int d^n \pi d^n\theta \delta(\sqrt{V(V+\bPi^2)}-E)
\end{equation}
 Integrating over $\bPi$ yields eq.\eqref{psvolIII}, concentrated near small $V$.
Near a minimum, $V$ is quadratic, giving (after an orthogonal diagonalization of its Hessian)
$
V\simeq V_I + \frac{1}{2}\sum_{i=1}^n m_{Ii}^2 (\theta_i-\theta_{Ii})^2
$.
Evaluating this for a basin $|\theta_{Ii}-\theta_i|<1/m_{Ii}$ at fixed $n$ yields (see \S\ref{app:ps-vol-details}) 
\begin{equation}\label{eq:vol-I-maintext}
\text{Vol}({\cal M}_I)\to b_n \left(\frac{2\pi^{n/2}}{\Gamma(n/2)}\right)^2\frac{E^{n-1}}{\prod_i m_{Ii}}\log(V_I) ~~~~ V_I\to 0, 
\end{equation}
where $b_n$ is a constant.
Since $V_I\to 0$ dominates in this formula, mixing with a small enough value of $t_\text{mix}$ is a sufficient condition for successful optimization in this framework.  Since the probability of being in a region of phase space is proportional to its volume, \eqref{eq:vol-I-maintext} yields an analytic prediction for the distribution of results over multiple low-lying minima which we test experimentally in \S\ref{sub: two-basins}.  Eq. \eqref{eq:vol-I-maintext}\ and the results \S\ref{sub: two-basins} exhibit a preference for low-lying flat minima; on these BBI evolves faster than GD \eqref{eq:gd-nomom-shallow}-\eqref{eq:bi-shallow}. Other ECD models, e.g. $V\to g(V)$, may enhance this effect.  In addition to its use for optimization, the formula enables reverse-engineering an ECD model whose trajectories sample from a desired distribution \cite{sampling}. 

A priori we do not know whether the optimization algorithm combined with the loss function $V$ constitutes a strongly chaotic, exponentially mixing system (although this may be typical), and if so what its value of $t_\text{mix}$ is.  Strong mixing arises in dynamical billiards, see e.g. \cite{Chernov06chaoticbilliards, 2017arXiv170102955S}.  Nearby trajectories bouncing off a billiard ball disperse, leading to sensitive dependence on initial conditions. Inspired by this,
we introduce randomized momenta at fixed $\bPi^2$ in our algorithm (cf. Sec. \ref{sec:algo}).\footnote{  Hamiltonian sampling \cite{2017arXiv170102434B, hanada2018markov} also involves randomized momenta, in general not conserving energy.}  

We stress that the utility of this result depends on the timescale $t_\text{mix}$.  Consider a system consisting of a box with an arbitrarily tiny hole to a much larger (or even infinite) domain beyond the box. Although the latter has more phase space volume, the time for a random trajectory to hit the hole may be arbitrarily long.  So the diverging phase space volume of a region $R$ is not by itself enough to guarantee that the system ends up in $R$ after a small timescale.  However, this analogy appears too pessimistic for our case:  the gradients introduce forces toward the minimum, and the speed limit on motion is weaker for larger loss $V$, so the system moves relatively quickly through such regions, spending more time near $V=0$.   Relatedly, if we endow frictionless momentum (non-relativistic particle motion) with bounces leading to chaos, its version of \eqref{eq:vol-I-maintext} is \emph{not} dominated by small $V$ \eqref{eq:volMom}.
This method is based on well developed statistical mechanics intuition and empirical observation, but does not come with a general theoretical guarantee, though it would be interesting to build from the continuum analysis \eqref{eq:gd-nomom-shallow}-\eqref{eq:bi-shallow} to seek guarantees in the convex setting though this may be complicated by the nonlinear updates.
Typically, the determination of $t_\text{mix}$ for a complicated system requires experiment, and we will test our intuition that way.

\subsection{The Effect of Noise (Minibatches)}\label{sec:minibatches}

Our algorithm will work similarly with or without the use of minibatches, for which the input is a subset of the data updated every $N_\text{batch}$ steps.  With minibatches, the new effect on the dynamics is that the potential $V(\bTh, t)$ becomes a function of time separately from its dependence on $\bTh$.  We retain our prescription of explicitly restoring the initial energy $E$ even when  minibatches are used; this is a deviation from the motivating physics model, for which explicit time dependence generically leads to a lack of energy conservation.  For a loss of the form $\Delta V+ V=\Delta V+\sum_\text{batches} F_\text{batch}$  with $F_\text{batch} \ge 0~~\forall$ batches, Theorem \ref{thm:not-stuck} continues to apply.  

We will explore aspects of the behavior of BI with mini-batches at greater length in \S\ref{app:stochastics}, both from this deterministic (but time-dependent) perspective, and from the perspective of appropriate averaged stochastic dynamics; the speed limit $\dot\bTh^2 < V$ constrains the variance of $\dot\bTh$.

\subsection{The BBI Algorithm}\label{sec:algo}

We now combine the ingredients introduced so far and describe our detailed optimization algorithm. As explained in Sec.~\ref{sec:continuum-dynamics}, the starting point is given by the the continuum equations \eqref{eom}. The first step is their discretization, which we perform using a first order \emph{symplectic method}.\footnote{\emph{Symplectic methods} are integration methods that preserve the geometric structure of Hamiltonian dynamics. They are relevant for our discussion since this also ensures that phase space volumes are preserved. See e.g. \cite{hairer2006geometric} for a review.} This ensures that at this stage volume density in phase space is preserved (cf. \cite{RGD2020}, Sec.~3 ), and results in the update rules
\begin{equation}\label{eq:update-pi}
\bPi_{k+1} =\bPi_k - \frac{1}{2}\bnabla V_k \Delta t\left(\frac{E}{V_k}+\frac{V_k}{E}\right)
\end{equation}
\begin{equation}\label{eq:update-th}
\bTh_{k+1} = \bTh_k+ \bPi_{k+1} \Delta t  \frac{V_k}{E}\;,
\end{equation}
where the subscript $k$ refers to the $k$-th iterative step, and $V_k = V(\bTh_k, t_k)$ (with the explicit $t$-dependence applicable in the case with minibatches discussed in \S\ref{sec:minibatches}\ and \S\ref{app:stochastics}). The energy $E$ is a constant defined as
\begin{equation}\label{eq:enIn}
E \equiv V_0+ \delta E   \;,
\end{equation}
where $V_0$ is the initial value of the (shifted) objective function \eqref{eq:SL-V-F-shift-def} at the beginning of the evolution and $\delta E\ge 0$ is an optional extra hyperparameter. Its primary role is to add an initial energy to overcome possible energy barriers higher than the initial value of the potential, which is useful in highly non-convex problems, such as the Ackley function discussed in Sec.~\ref{sub: ackley}. It can also introduce more chaos in the system to help distribute the results.

As shown in Theorem \ref{thm:not-stuck}, the dynamics just described stops changing $\bTh$ appreciably only when $V=F-\Delta V\to 0$ where $\Delta V$ is the hyperparameter defined in \eqref{eq:SL-V-F-shift-def}.  
This acts both as threshold for the required accuracy of the solution and as a shift in case the expected optimum $F(\bTh)\neq 0$.  

 In fact an adaptive tuning of $\Delta V$ is possible, added as an option to the altorithm and tested in computational chemistry experiments \cite{chemistry}. Given a too-high initial guess, the loss extends to $V = F-\Delta V <0$ and the trajectory will jump to a small negative value $V<0$ due to the discreteness. Conditioned on this, $\Delta V$ may be lowered; this iteratively tunes it.

The vector $\bTh$ is initialized by choosing an initial point in parameter space (e.g. a standard initialization of a Neural Network (NN)). The vector $\bPi$ can be initialized in various ways consistent with energy conservation. A natural choice is to initialize it in the direction of (minus) the initial gradient,
\begin{equation}
    \bPi_0 \equiv -\frac{\bnabla V(\bTh_0)}{| \bnabla V(\bTh_0|} \sqrt{\frac{E^2}{V_0}-V_0},
\end{equation}
with norm satisfying \eqref{eq:Econs}. Another possibility is to initialize it in a random direction.  Notice that if $\delta E = 0$,  $\bPi_0$ is always initialized to zero, as is common in optimization.

The optimizer described so far implements the BI dynamics, but does not yet contain the extra features encouraging chaotic mixing as discussed in \S\ref{sub:mixing}.   This can be achieved by introducing random bounces, which conserve energy by keeping $\bPi^2$ the same in the process. Such bounces can be easily implemented by generating a new random $\bPi$ with the same norm as the original $\bPi$.  We implement these in two ways:
(i) $N_b$ fixed bounces separated by $T_0$ timesteps, and (ii) A variable number of progress-dependent bounces that are performed if for $T_1$ timesteps there is no progress in the search of the minimum, e.g. if it fails to reach a value of $V$ smaller than all those previously seen.
We will refer to this Bouncing BI dynamics as \emph{BBI}.

As a final ingredient, we include a rescaling of $\bPi$ that restores possible energy conservation violations due to numerical and discretization effects and minibatch transitions.\footnote{Recall from \S\ref{sec:continuum-dynamics} that the continuum dynamics conserves energy exactly in the noise-free case.}
This is performed by computing the current value of $\bPi_k^2$ at each step and comparing it with $ \sqrt{ \frac{E^2}{V_k}-V_k }$, the value it should have at step $k$. If there is a discrepancy, we homogeneously rescale all the components of $\bPi_k$ to restore $\bPi_k^2$ to the appropriate value. The rescaling is not performed if the quantity in the square root happens to be negative, something that could happen very early in the evolution due to small numerical errors.
This mechanism, not entirely faithful to the continuum dynamics, introduces small changes in the phase space measure. We expect this effect to be small and not appreciably affect the analytic estimates,  confirming this empirically in quantitative experiments in Sec.~\ref{sub: two-basins}.
Algorithm \ref{alg:BBI} summarizes BBI.
\begin{algorithm}[htb!]
   \caption{\emph{BBI}. $\varepsilon_1$ and $\varepsilon_2$ are numerical constants ensuring stability (good default values are $\varepsilon_1 = 10^{-10},$ $\varepsilon_2 = 10^{-40}$). The hyperparameter $\Delta V$ can be adaptively tuned as explained in the main text. See also the code for a preliminary working implementation. }
   \label{alg:BBI}
\begin{algorithmic}
   \REQUIRE $\Delta t$: Stepsize
   \REQUIRE $\Delta V$: Shift hyperparameter.
   \REQUIRE $\delta E$: Extra initial energy.
   \REQUIRE $T_0, T_1, N_b$: Bouncing thresholds and \# of fixed bounces.
   \REQUIRE $F(\bTh)$: Function to minimize.
   \REQUIRE $\bTh_0$: Initial parameter vector.
    \STATE $\{c_0, c_1, n_b\} \gets \{0,0,0\}$ (Initialize counters for bounces)
    \STATE $V\gets F(\bTh_0) - \Delta V$ (Initialize $V$)
    \STATE $E \gets V +\delta E$ (Initialize energy)
    \STATE $\bPi_0 \gets - \frac{\bnabla F (\bTh_0)}{|\bnabla F (\bTh_0)|} \sqrt{\frac{E^2}{V}-V}$ (Initialize momenta)
    \STATE $t\gets 0$ (Initialize timestep)
\WHILE{$V> \varepsilon_2 $}
    
    \IF {$c_0 \neq T_0 $ \AND $c_1 \neq T_1$}
    
    \STATE $t\gets t+1$
    \STATE $\pi_C^2 \gets V(\frac{E^2}{V^2}-1)$ (Correct value of $\bPi^2$)

    \IF{ $|\bPi_{t-1}^2-\pi_C^2|< \varepsilon_1$ \OR $\pi_C^2< 0$ }
    \STATE $\alpha \gets 1$
    \ELSE
    \STATE $\alpha \gets \sqrt{\frac{\pi_C  ^2}{\bPi_{t-1}^2}}$ (Factor to restore energy cons.)
    \ENDIF
    
    \STATE $\bPi_{t} \gets \alpha \bPi_{t-1}$ (Restore energy conservation)
    
    \STATE $\bPi_{t} \gets \bPi_t-\frac{1}{2} \Delta t (\frac{V}{E}+\frac{E}{V})\bnabla F(\bTh_{t-1})$ 
    \STATE $\bTh_{t} \gets \bTh_{t-1}+\Delta t \frac{V}{E} \bPi_{t}$
    
    \STATE{$c_0\gets c_0+1$}
    \STATE{$c_1\gets c_1+1$}
    
    \IF{\emph{new minimum}}
    \STATE $c_1\gets 0$
    \ENDIF
    \STATE $V\gets F(\bTh_t) - \Delta V$
    \ELSE
    \STATE $\bPi_N\gets \emph{random vector}$ (Generate random vector)
    \STATE{$\bPi_t \gets \bPi_N \sqrt{\frac{\bPi_t^2}{\bPi_N^2}}$} (Bounce momenta cons. $\bPi^2$)
    \IF{$c_0 = T_0$}
    
    \STATE $n_b\gets n_b+1$
    \IF{$n_b < N_b$}
    \STATE{$c_0 \gets 0$}
    \ELSE
    \STATE $c_0\gets c_0+1$
    \ENDIF
    \ENDIF
    \STATE $c_1 \gets 0$
 \ENDIF
 \ENDWHILE
\end{algorithmic}
\end{algorithm}
\section{Experiments}
In this section, we empirically confirm the theoretically predicted behavior of BBI as an example of (S)ECD, and compare it to (S)GDM, studying both the noise-free and noisy (stochastic) cases.  We do not systematically compare to adaptive optimizers such as Adam \cite{kingma2017adam}, leaving such features in BBI to future work.  In \S\ref{sub: ackley}-\ref{sub: two-basins} we analyze synthetic functions useful for testing optimization algorithms, illustrating the theory derived in previous sections and testing BBI against GDM.  In \S  \ref{sec:PDEs}-\ref{sec:ML-problems} we present richer, high-dimensional experiments in PDE solving and ML (MNIST \cite{mnist} and CIFAR \cite{cifar}).  These include minibatches, enabling us to confirm the robustness of BBI in that setting.

We ran the synthetic experiments and MNIST on standard laptop CPUs, while for CIFAR and the PDEs we used two GPUs. More details, including the full source code and sample results, can be found at \href{https://github.com/gbdl/BBI}{https://github.com/gbdl/BBI}.

\begin{figure*}[!ht]
\centering
\includegraphics[width=0.25\linewidth]{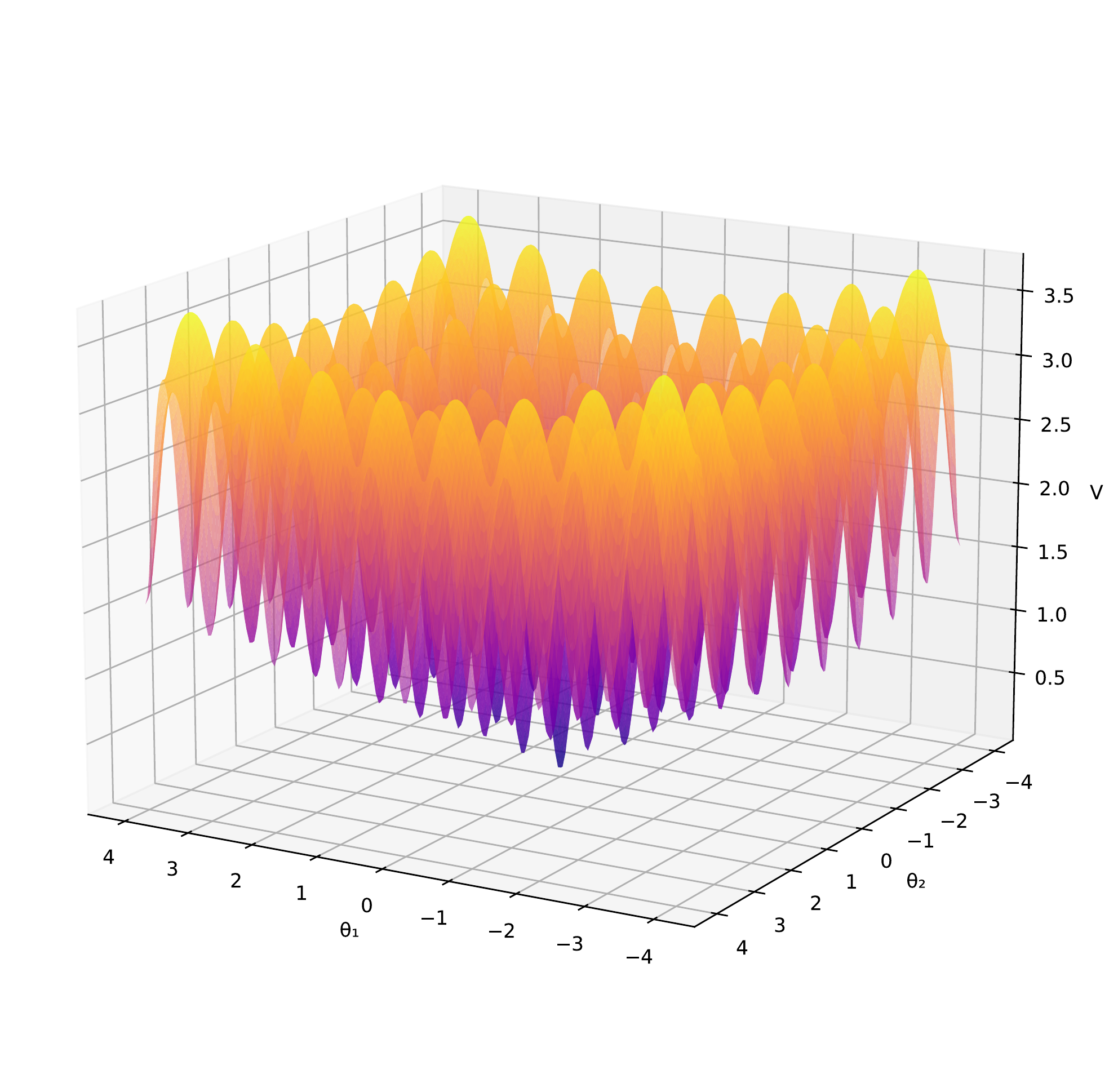}
\hfill
\includegraphics[width=0.25\linewidth]{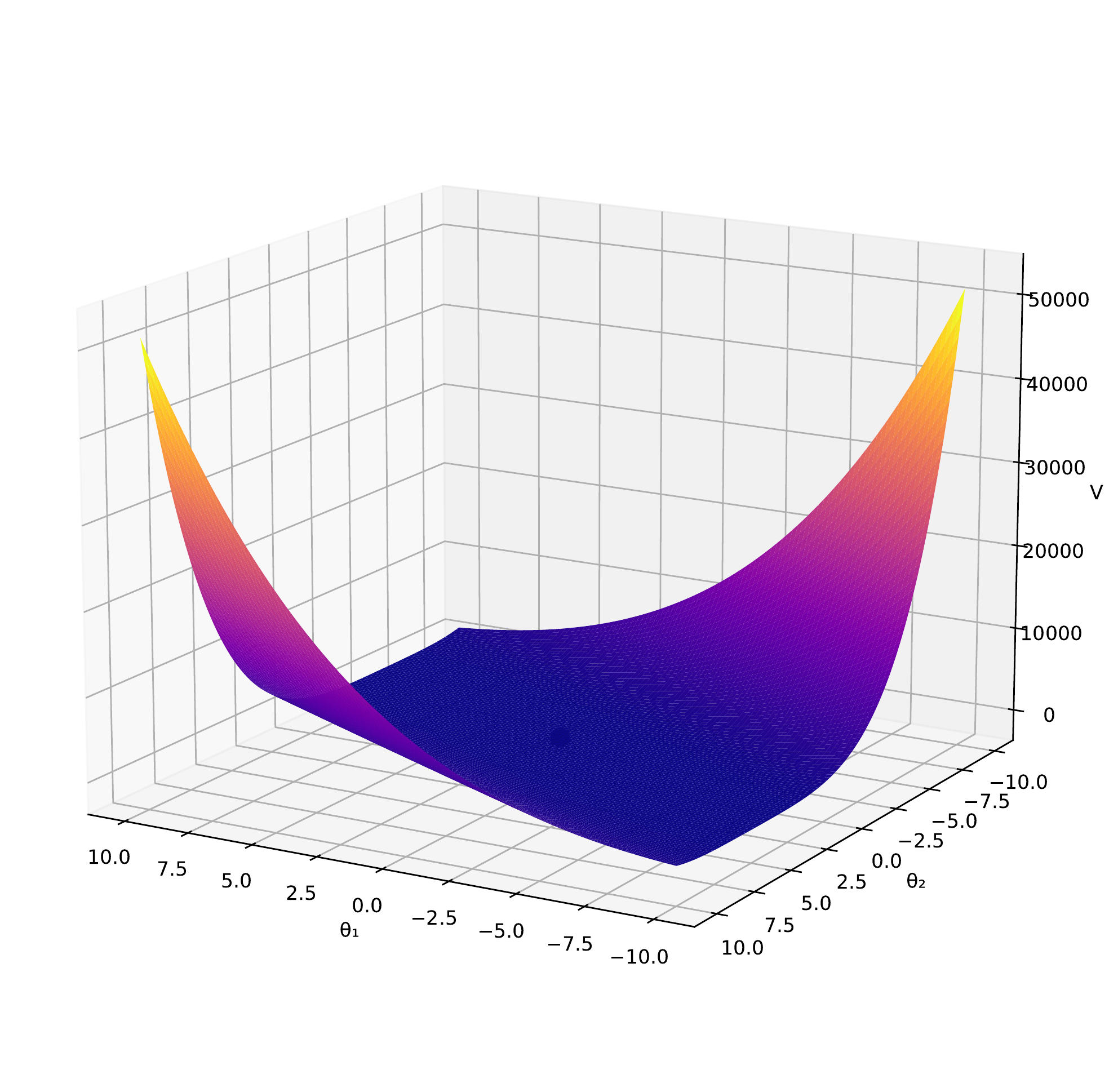}
\hfill
\includegraphics[width=0.25\linewidth]{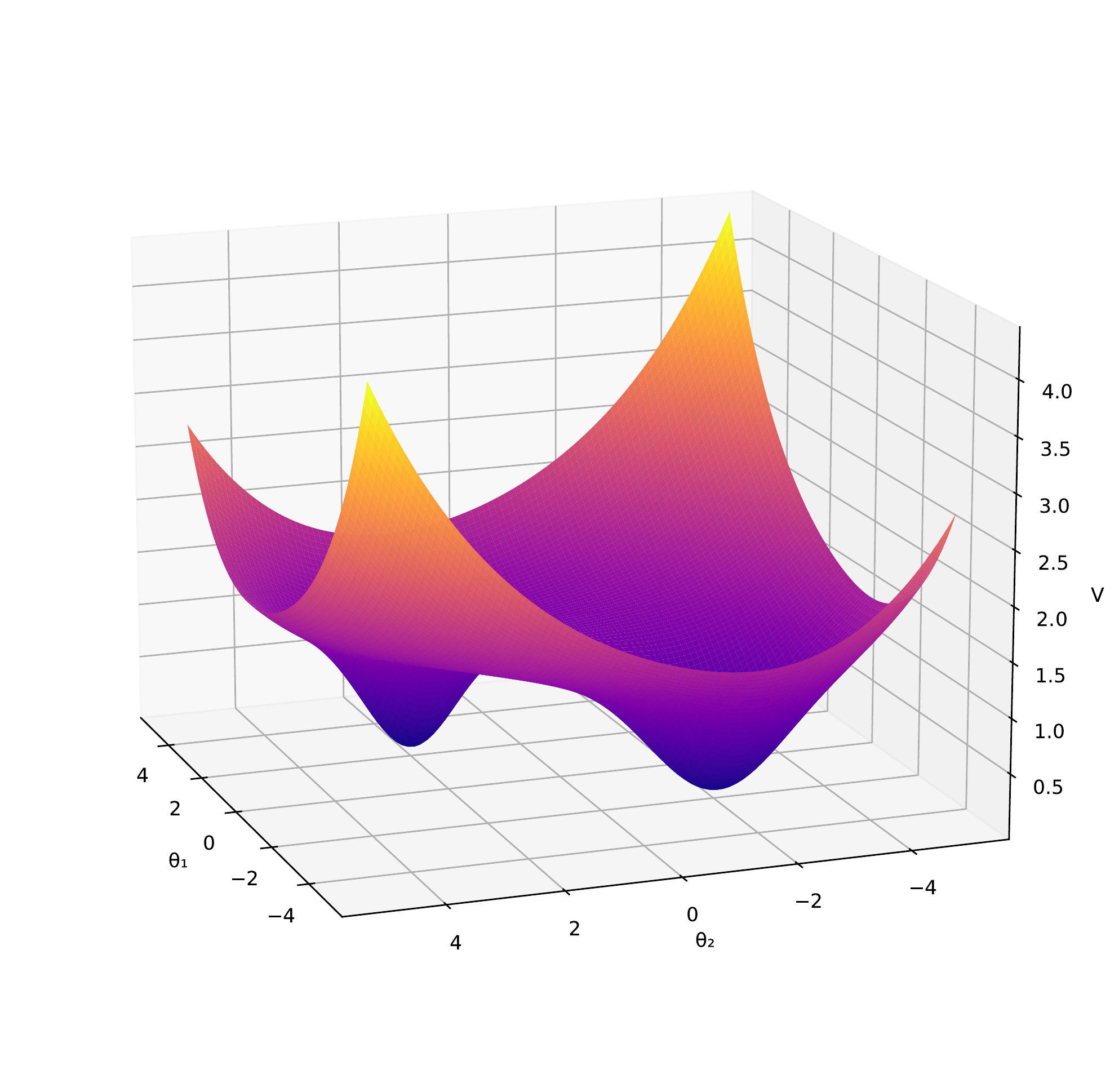}
\hfill
\caption{Synthetic benchmark functions (depicted in 2 dimensions). A highly non-convex function (left), a function with shallow valleys (center) and a simple multi-basin function (right).}
\label{fig:synthetic}
\vskip -0.2in
\end{figure*}

\subsection{Highly Non-Convex Landscapes.}\label{sub: ackley}

In this section we consider highly non-convex functions, focusing on the hard-to-optimize Ackley function \cite{ackley}, plotted in Fig.~\ref{fig:synthetic} (left) and defined as
\begin{multline} \label{eq:ackley}
    F(\theta_1, \theta_2) \equiv -20 \exp\left[-0.2\sqrt{0.5\left(\theta_1^{2}+\theta_2^{2}\right)}\right]+\\ 
 -\exp\left[0.5\left(\cos 2\pi \theta_1 + \cos 2\pi \theta_2 \right)\right] + e + 20\,.
\end{multline}

We test the BBI optimizer on this function, to check its ability to escape the many local minima and to compare it with friction-based methods. A typical evolution is displayed in Fig.~\ref{fig:ackley-2d-BBI}.
Starting from a fixed random point, we observe:
\begin{itemize}
    \item If $\delta E$ is not large enough to escape the initial local minimum, the particle keeps oscillating indefinitely. This is consistent with conservation of energy.
    \item When $\delta E$ is above a certain threshold, the particle escapes the first local minimum and keeps moving, eventually discovering the global minimum at $\theta_i=0$.
\end{itemize}

\begin{figure}[ht]
\vskip 0.2in
\begin{center}
\centerline{\includegraphics[width=.65\columnwidth]{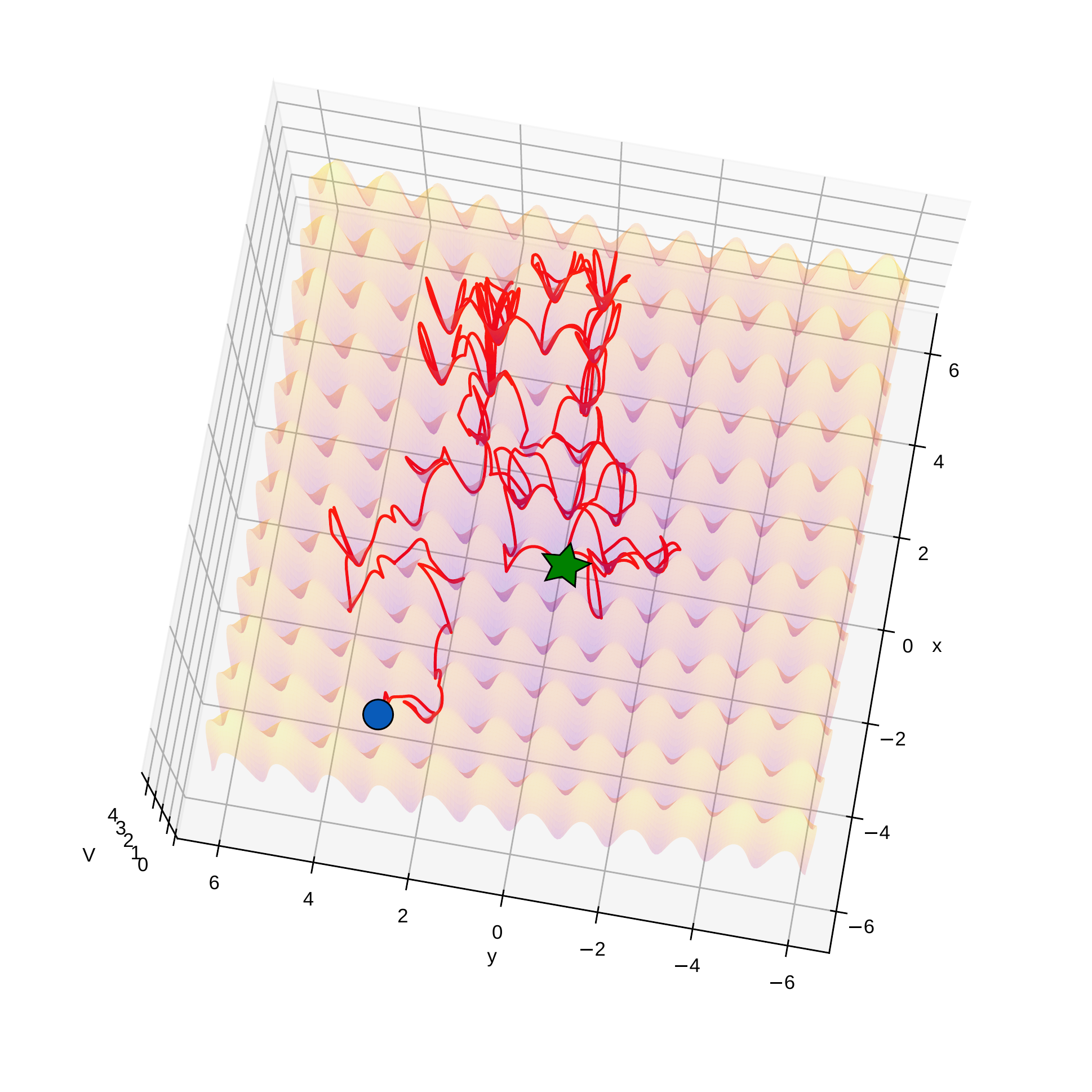}}
\caption{A typical BBI trajectory (starting from the circle) explores the rich landscape, stopping only at the global minimum (star).}
\label{fig:ackley-2d-BBI}
\end{center}
\vskip -0.2in
\end{figure}

This is to be contrasted with friction based optimizers, such as GDM. 
In that case, we observe three regimes:
(i) with a small stepsize, the evolution is not able to escape from the first local minimum it finds, (ii) with higher stepsizes, the optimizer jumps out of the first local minimum and keeps exploring the landscape erratically. Doing so it might be able to get somewhat close the the global minimum, but the evolution is unstable and does not converge to that point, (iii) with even higher learning rates the evolution is divergent.
This behavior is consistent with the analysis in \cite{catapult}, where the second phase is called \emph{catapult phase}. We will adopt the same name\footnote{ \cite{catapult} found that on certain problems with flat minima the catapult phase is best for (S)GD. See Sec.~\ref{sub: zakharov}\ for a function with shallow regions.}.

To confirm this behavior in general we performed a search in the hyperparameter space by using hyperopt with its Tree Of Parzen Estimators algorithm \cite{hyperopt}.
To do so, we picked an initial point, and used hyperopt to find the best hyperparameters for the various optimizers.
 For each optimizer we ran 500 trials, evolving for 200 steps each, using 
the lowest point found during this evolution as the measure of success.
For GDM we hyperoptimized both the learning rate $\eta\in [10^{-4},.5]$ and the value of the momentum $\mu\in [.0,1.0]$.\footnote{We also ran hyperopt a second time on the pair $\{\eta, \gamma = -\log{\mu}\}$, not finding an appreciable difference (in other experiments this improved GD e.g. compared to \cite{RGD2020}). For (S)GDM and its parameters we are using the default implementation in Pytorch.}
For BBI, we fixed a default value for the chaos-inducing hyperparameters ($T_0 = 20, N_b = 4, T_1 = 100$), set $\delta E = 2$ to overcome the initial barrier, and used hyperopt only on the step size. 
The results confirm the observations above. Indeed, for GDM to escape the initial minimum, hyperopt must increase the step-size, which then results in an erratic unstable evolution. Since hyperopt tries to avoid divergences, it will automatically select the catapult phase to escape the initial minimum and get to a lower loss. To confirm this interpretation, we also ran hyperopt again on GDM but with a more stringent upper bound on the allowed value of step-size, finding that with this constraint it is not able to escape from the initial local minimum.

After this phase of estimation of the best hyperparameters, we ran the evolution again for a longer time. 
GDM in the catapult regime keeps moving erratically, while BBI explores the rich landscape eventually finding the global minimum.
Fig.~\ref{fig:hyperopt-ackley} summarizes these results. 
\begin{figure}[ht]
\centering
\begin{tikzpicture}[x=\textwidth,y=\textwidth, every node/.style = {anchor=north west}]
\node[anchor=center] at (0.22, -0.22) {\includegraphics[width=0.8\columnwidth]{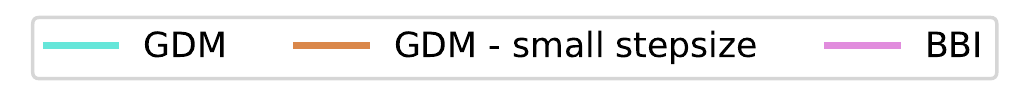}};
\node at (0, 0) {\includegraphics[width=.9\columnwidth]{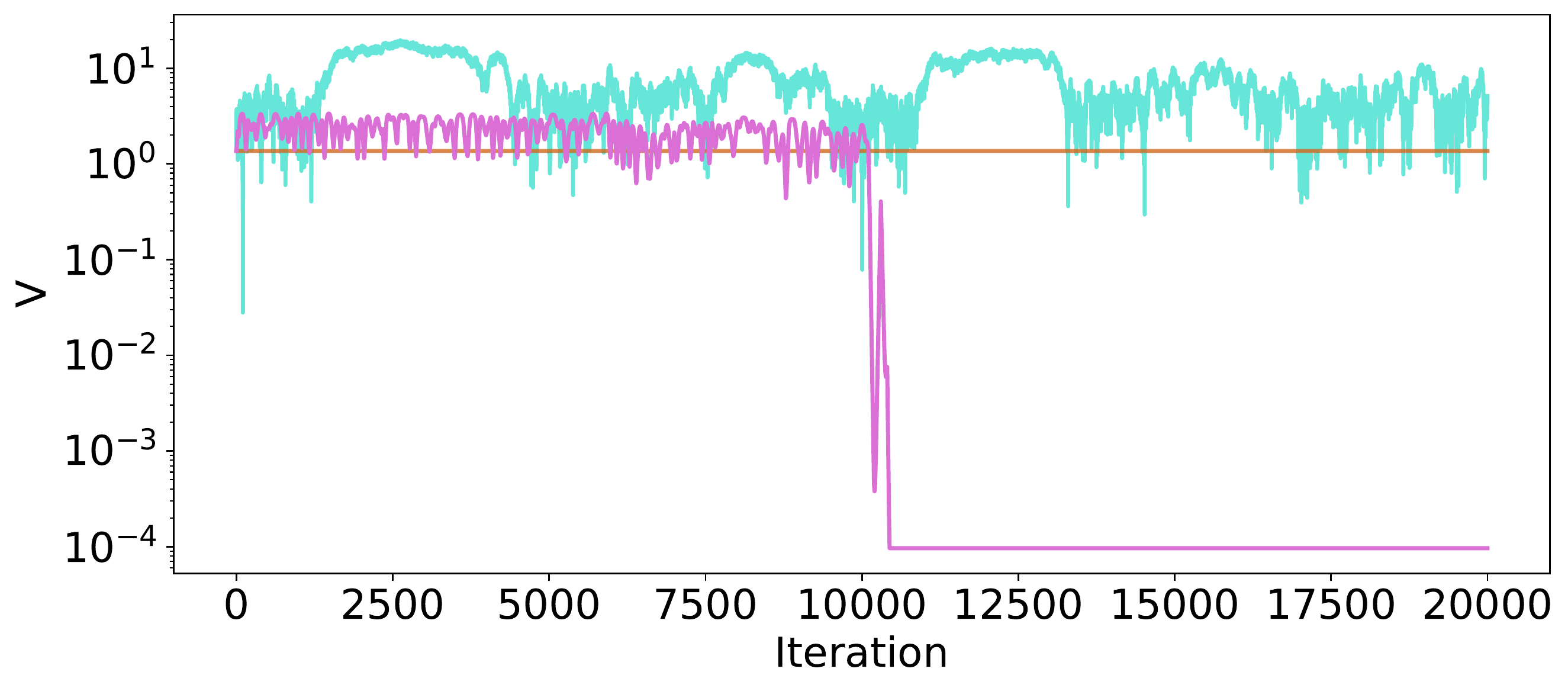}};
\end{tikzpicture}
\caption{Comparison of the optimizers on the highly non-convex Ackley function \eqref{eq:ackley}. The plot shows a typical run with the best hyperparameters as estimated with hyperopt. GDM is either stuck on the initial local minimum or it catapults itself out with a high step size, which however forbids it to converge to the global optimum. BBI explores the rich landscape eventually finding the global minimum and converging to a prescribed accuracy specified by $\Delta V$ (here $10^{-4}$).}
\label{fig:hyperopt-ackley}
\vskip -0.2in
\end{figure}

Finally, we also checked that these results are robust against changes of the initial point. 
To confirm this, we evolved from 100 randomly selected points for $(x,y)\in [-4,4]^2$ using the hyperparameters estimated above. For each of the points we ran the evolution 5 times, to take advantage of the chaotic (and non-deterministic) features of BBI, but not during the estimation phase. Doing this, we obtained that for 90/100 random initial points, BBI achieved $V<5\times10^{-4}$ within $3\times10^4$ iterations, compared to 2/100 of the times where GDM gets close to the minimum by chance during its erratic evolution.
 
\subsection{Shallow Valleys}\label{sub: zakharov}
In this section, we analyze the case where the global minimum lies in an almost-flat valley, making efficient convergence hard.  For concreteness, we focus on the Zakharov function (\cite{jamil2013literature}, Function 173), plotted in 2 dimensions in Fig.~\ref{fig:synthetic} (center), with $n$-dimensional definition
\begin{equation}
    F(\bTh)\equiv\sum_{i=1}^n \theta_i^2+ \left(\frac12 \sum_{i=1}^n i \theta_i\right)^2+\left(\frac12 \sum_{i=1}^n i \theta_i\right)^4\;.
\end{equation}
It has no local minima, but the global minimum at $\bTh=0$ lies in a nearly flat valley, slowing optimization. 
We take $n = 10$, and compare the performance of BBI and GDM with the same methodology described in Sec.~\ref{sub: ackley}. That is, we start from a fixed reference point, $(-1,...,-1)$,  and use hyperopt to estimate the best hyperparameters by evolving 500 times for 2500 iterations. For GDM  we search on the stepsize twice ($\eta \in [10^{-10}, .5]$ and $\eta \in [10^{-10}, 10^{-5}]$) and on the momentum $\mu \in [.0, 1.0]$, while for BBI we set high thresholds to turn off bounces,
set $\delta E = 0$, and search only on the step-size $\Delta t \in [10^{-6}, 10^{-2}]$.  With the best hyperparameters thus determined, we ran the evolution for $10^4$ iterations, obtaining the results in Fig.~\ref{fig:hyperopt-zakharov}.
\begin{figure}[ht]
\centering
\begin{tikzpicture}[x=\textwidth,y=\textwidth, every node/.style = {anchor=north west}]
\node[anchor=center] at (0.24, -0.22) {\includegraphics[width=0.4\columnwidth]{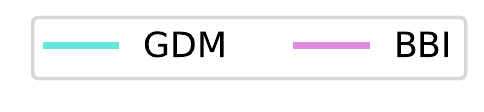}};
\node at (0, 0) {\includegraphics[width=.9\columnwidth]{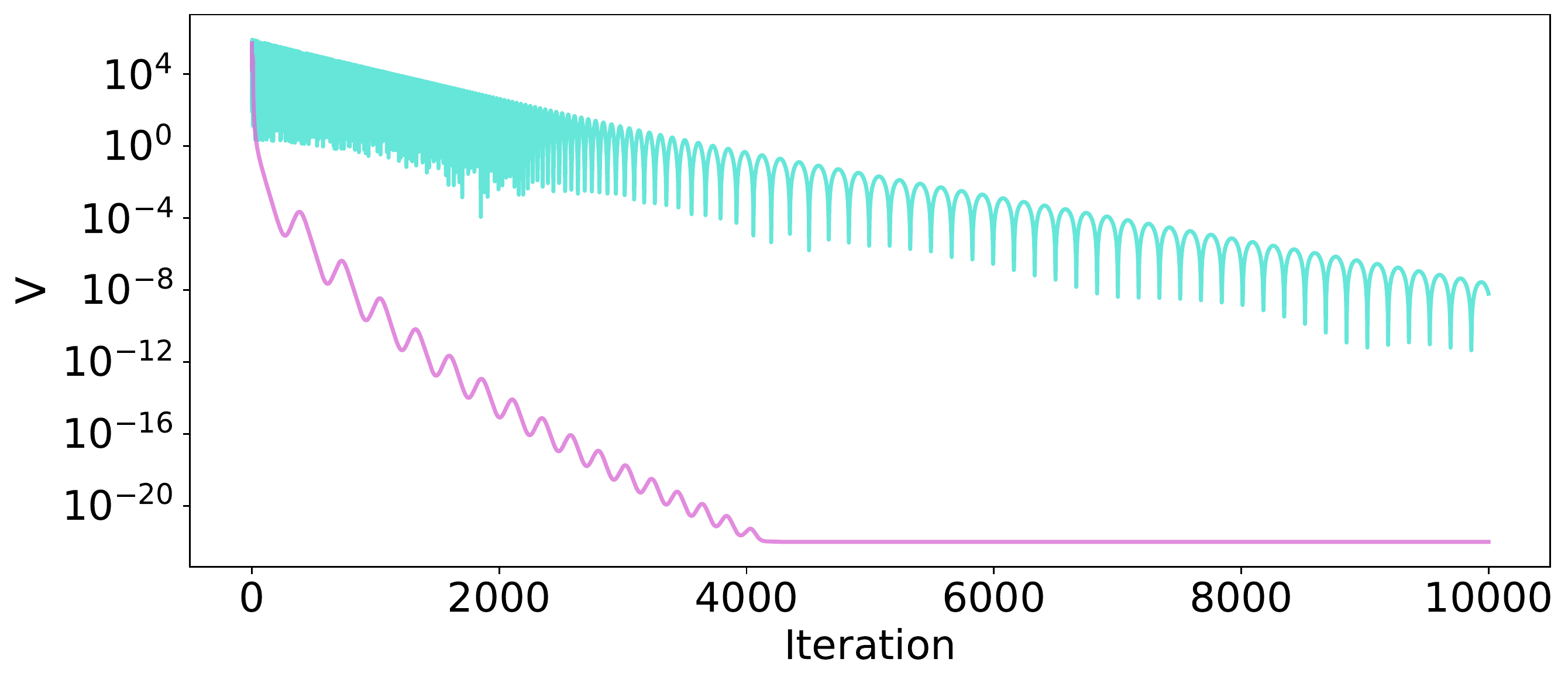}};
\end{tikzpicture}
\caption{Performance comparison on a ten-dimensional Zakharov function with the best estimated parameters. BBI proceeds faster and stops at the desired accuracy specified by $\Delta V$, here $10^{-22}$.}
\label{fig:hyperopt-zakharov}
\end{figure}
We also selected random initialization points and ran multiple evolutions with the previously estimated best hyperparameters, confirming that BBI consistently performs better than GDM.
\subsection{Two-Dimensional Multi-Basin Function}\label{sub: two-basins}

In this section we study the optimization in a simple landscape with two global minima whose basins of attraction have different shapes.  In this landscape, BI alone (without bounces) does not exhibit chaotic mixing in our experiments. We will confirm empirically that the chaotic billiards-inspired bouncing prescription in BBI is able to find multiple solutions of global optimization problems which matches the predictions from mixing dynamics.
Explicitly, we consider the function 
\begin{multline}
    \label{eq:basins-formula}
    F(\bTh) \equiv - \textrm{exp}\left(-0.4 |\bTh-\mathbf{c_1}|^2 \right)+\\
    -(1-\epsilon)\textrm{exp}\left(-0.8 |\bTh-\mathbf{c_2}|^2 \right)+\\
    +10^{-3} |\bTh-\mathbf{c_1}|^2|\bTh-\mathbf{c_2}|^2+1    \;,
\end{multline}
where $\epsilon\ll 1$ can be tuned to ensure the two minima are numerically at the same height, and $\mathbf{c_1}$ and $\mathbf{c_2}$ are the locations of the two global minima. We will use $\epsilon \sim 10^{-6}$, $\mathbf{c_1} = (-2,-2)$, $\mathbf{c_2} = (2,2)$.  
This function has two global minima for which $V = 0$ (see Fig. \ref{fig:synthetic} (right)). 
\begin{figure}[ht]
\vskip 0.2in
\begin{center}
\centerline{\includegraphics[width=0.5\columnwidth]{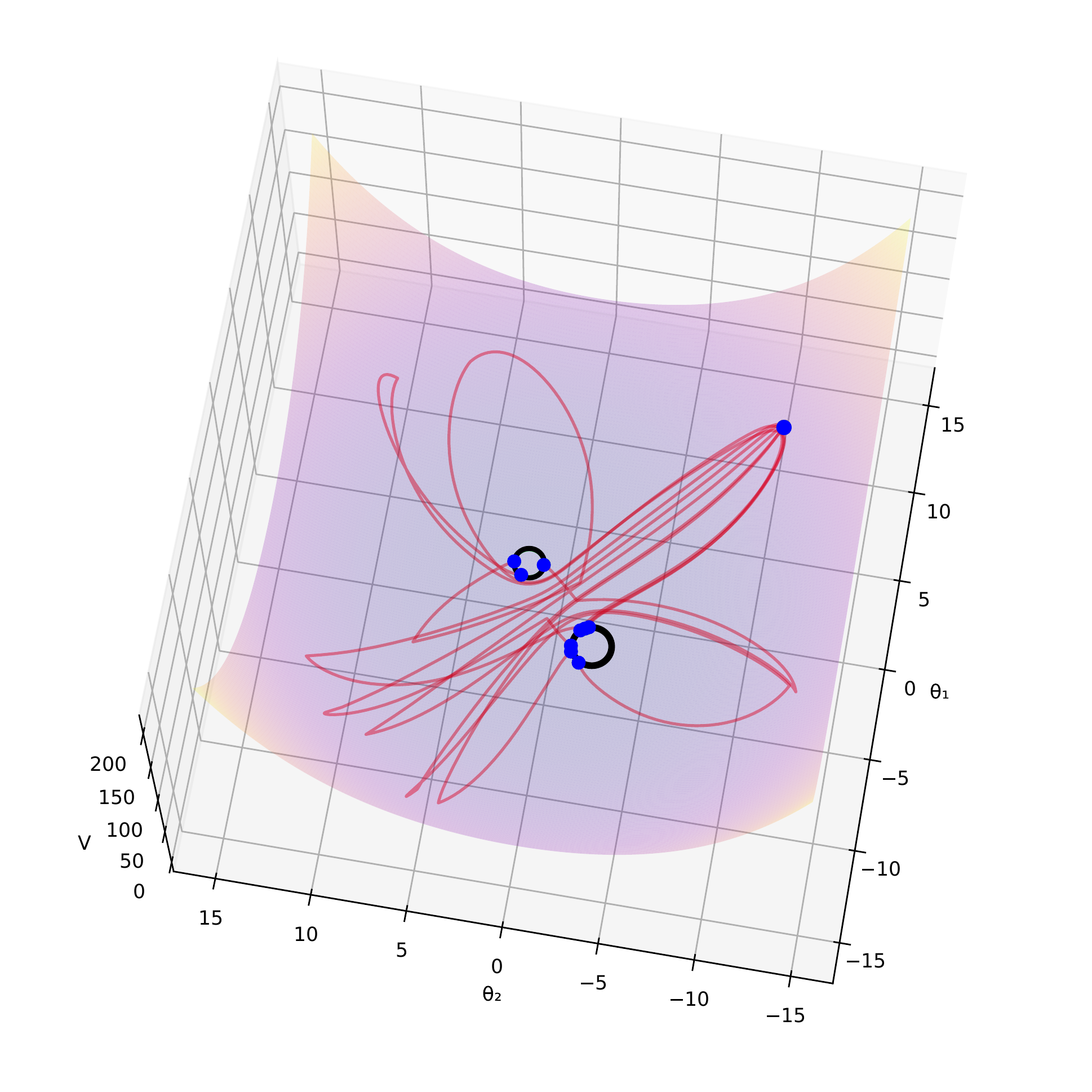}}
\caption{Ten trajectories starting from the same initial point, bouncing at the beginning. The black circles denote boundaries of the regions of basins. This shows that an initial random bounce and a small number of other bounces along the evolution are able to distribute the trajectories in the different basins.}
\label{fig:sample-traj-basins}
\end{center}
\vskip -0.2in
\end{figure}

The shapes of the basins directly enter in the volume formula \eqref{eq:vol-I-maintext} through the eigenvalues of the Hessian $H_{i j} \equiv \partial_{i j}^2 V$, $i,j = \{1,2\}$. For the potential (\ref{eq:basins-formula}) this gives
the ratio of the volumes of the two basins
\begin{equation}\label{eq:basins-ratio-th}
    {\textrm{Vol}(\mathcal{M}_1)}/{\textrm{Vol}(\mathcal{M}_2)} = 
    \sim 1.93\;,
\end{equation}
assuming the energy is held fixed. In a perfectly mixing system this ratio is also the ratio of convergence to the two basins. To test how close the bouncing prescription in the BBI algorithm brings the potential \eqref{eq:basins-formula} to a mixing system, we performed multiple evolutions and checked the ratio of times the particles get in one of the two basins. 
With $10^3$ evolutions, the partial ratios already asymptote to a value close to the theoretically predicted value (\ref{eq:basins-ratio-th}). For the example in Fig. \ref{fig:basins-partial-ratios}, the agreement is within $\sim 11\%$.
\begin{figure}[ht]
\vskip 0.2in
\begin{center}
\centerline{\includegraphics[width=\columnwidth]{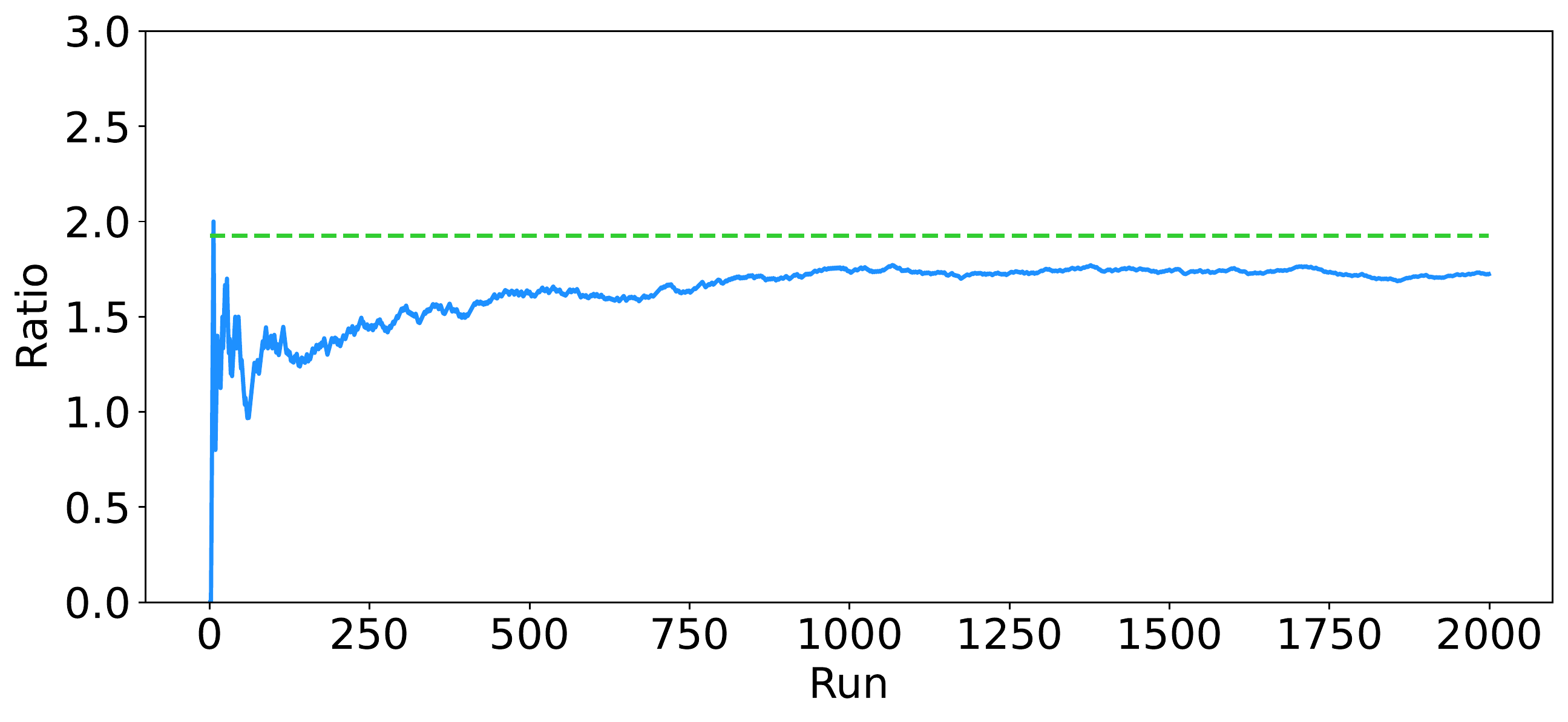}}
\caption{Partial ratios of convergence to the different basins on multiple runs (blue). All the runs start from the same initial point and are performed with $N_b = 1, \Delta t = 10^{-2}, T_1 = 750, T_2 = 20, \delta E = 0, \Delta V = 10^{-3}$. The asymptotic ratio is  within 11\% of the expected ratio for a mixing system (Eq. \eqref{eq:basins-ratio-th}, green). }
\label{fig:basins-partial-ratios}
\end{center}
\vskip -0.4in
\end{figure}

This shows that the addition of bounces introduces chaos even in a simple non-mixing potential for pure BI.  This implies that BBI can find multiple solutions of optimization problems, a result we will reproduce in \S\ref{sec:PDEs} in a PDE-solving example.
The analytic estimate for the ratio of trajectories landing on different minima \eqref{eq:vol-I-maintext}\ holds up in this simple experiment.  The formula \eqref{eq:vol-I-maintext} assumes effective mixing on a relevant timescale and requires knowledge of the geometry of the basins near each $V=V_I$.   
Since these are not a priori known in many problems, it might be hard to check for some situations. In any case, the present experiment shows that the bounces can produce approximately mixing systems such that the optimization converges to the different solutions with good control.  Here we make no comparison to GD-like optimizers, which lack a prediction of the ratio of convergence to different basins.

\subsection{PDE Solving Examples}\label{sec:PDEs}

In this section, we start the investigation of larger optimization problems, solving Partial Differential Equations (PDEs) with NNs.  PDEs are ubiquitous in science; developing efficient numerical solvers is of paramount importance. Many methods have been developed to attack this problem, including recent exploration of Machine Learning methods. (A partial list includes  \cite{ml-pde1, ml-pde2,ml-pde3,ml-pde4,ml-pde5,ml-pde6,ml-pde7,ml-pde8,ml-pde9} and \cite{ml-pde10,ml-pde11}.)
These promise more versatility, potentially avoiding the \emph{curse of dimensionality}, the exponential increase in complexity with the number of dimensions.  

The most common approach is to employ a NN as an ansatz for the solution of the PDE, using the squared PDE itself (together with its boundary conditions (BC)) as a loss function, schematically $V(\bTh)=F(\bTh)$ is given by:
\begin{equation}\label{eq:lossPDE}
    \sum_{x\in \mathcal{D}} \textrm{PDE}[\mathcal{N}(x; \bTh)]^2+\gamma \sum_{x\in \partial\mathcal{D}} \textrm{BC}[\mathcal{N}(x; \bTh)]^2+ \textrm{R}(\bTh)\,
\end{equation}
where $\mathcal{N}$ is a neural network, $\gamma$ is a fixed coefficient, R is a possible regularization term and $\mathcal{D}$ and  $\partial\mathcal{D}$ are sets of points randomly sampled from the domain of the PDE and its boundary.  We note that here the $\Delta V$ parameter \eqref{eq:SL-V-F-shift-def} is not needed and is set to 0.  
The price paid by these methods is the absence of guarantees about the convergence and accuracy. Indeed, even though NNs can represent large classes of functions with good accuracy \cite{approx-1, approx-2,approx-3, approx-4} , the fact that the training converges to those is not guaranteed (see e.g.~\cite{Adcock2021TheGB} for a recent analysis of convergence of Neural Networks to known functions), and
the optimizer is important.

In this regard, the BBI properties of not stopping until the best approximation to the solution is found and of predictably exploring different regions of the parameter space to capture multiple solutions appear very promising.  We have verified this by constructing a hard PDE problem, with two known distinct solutions, finding that BBI with a residual NN ansatz in \eqref{eq:lossPDE} successfully converges to both solutions with good accuracy consistently across hundreds of experiments.  See App.~\ref{app: PDE} for details on the PDE problem.

The randomly sampled points $x$ can be changed during the evolution, similarly to minibatches in ML problems, giving rise to a time-dependent potential $V(\bTh, t)$ as described in \S\ref{sec:minibatches}.  Doing this every epoch =1000 iterations, with $\Delta V=0,N_b=10, \Delta t= 5\times 10^{-6}, \delta E= 2\times 10^6, T_0=1000$, and $T_1=5000 $ for a small sample of 12 1500-epoch runs yields consistent solutions of the PDE, dominated by one solution (finding the other solution 1 out of 12 times).  Although small statistics, this illustrates BBI's robustness against batch-induced noise.

Indeed in BBI the effects of noise are limited by the speed limit on the variance of $\dot\bTh$ (see \S\ref{sec:minibatches} and \S\ref{app:stochastics}).  
We find multiple PDE solutions due to the bouncing (without noise) from the same initial point.  It would be interesting to quantify this analysis (c.f. \S\ref{sub: two-basins}) for these larger problems.

\subsection{Small-Scale ML Benchmarks}\label{sec:ML-problems}
Next we consider MNIST \cite{mnist} and CIFAR-10 \cite{cifar} as small ML benchmarks on which to test BBI, giving another check of whether the prescription of enforcing energy conservation described in Sec.~\ref{sec:algo} works well with minibatches (Table \ref{tab:ML}).

For MNIST we used a simple CNN having two convolutional layers with maxpool in between and a final dense layer. We took batches of size 50 and performed a grid search on the hyperpameters running for 3 epochs. For SGD we searched on the learning rate $\eta \in \{.001, .01, .05, .1, .2, .3\}$ and on the momentum $\mu \in \{ .85,.9,.95,.99,.999\}$,  while for BBI we scanned over the same learning rates $\Delta t \in \{.001, .01, .05, .1, .2, .3\}$ and on  $\delta E \in\{ .0, .1, .5, 1.0, 2.0 \}$, keeping fixed $\Delta V = 10^{-6}, T_0 = 100, N_b =5, T_1 = 1000$.
Choosing the best parameters according to their test accuracy, we performed 60 evolutions for 50 epochs, recording mean and median of the accuracy on the test set.

For CIFAR-10, we used a ResNet-18 \cite{resnet-18} and employed hyperopt to estimate the best hyperperameters by performing 50 trials for 3 epochs each. For SGD we searched both on $\eta \in [ 0.001, 0.2]$ and $\mu \in [ .8, 1.0]$ while for BBI we fixed the chaos-inducing parameters ($N_b = 100, T_0 = 100$, $T_1 =200$) and $\Delta V = \delta E = 0$ and only searched on $\Delta t \in [.001, .2]$. After the estimation phase, we evolved 16 times for 150 epochs with the best hyperparameters and computed mean and median. Both benchmarks use a standard Cross Entropy loss.
Given the small statistics, the distributions of results are not Gaussian and we only report mean and median.  
\begin{table}[t]
 \caption{Test set accuracy: \mbox{Mean , Median}.}
 \label{tab:ML}
 \vskip 0.15in
 \begin{center}
 \begin{small}
 \begin{sc}
 \begin{tabular}{lcc}
 \toprule
 Data set & SGD & BBI \\
 \midrule
  MNIST&    99.166 , 98.160 & 99.177 , 99.190 \\
CIFAR-10   & 92.628 , 92.655& 92.434 , 92.435\\
 \bottomrule
 \end{tabular}
 \end{sc}
 \end{small}
 \end{center}
 \vskip -0.1in
 \end{table}

We stress that the experiments in this section are not meant as a thorough comparison, but only as a first check that the BBI algorithm performs well on larger scale problems and is robust to the presence of minibatch. This also shows that overfitting does not appear to be a problem -- as anticipated given from the speed limit strongly slowing evolution before literally $V=0$ -- since the accuracy results on the test set are in line with SGDM. In progress are extensions to larger ML problems, analyzing ECD dynamics \eqref{eq:vol-I-maintext} in relation to properties of the loss $F(\bTh)$ \cite{imagenet}.

\section{Discussion}\label{sec:discussion}

We have established that maintaining energy conservation in the descent of the loss function $V$ is compatible with successful optimization, leading to certain analytic handles on the process that are unavailable for frictional methods (SGDM etc.).  Small-scale and benchmark numerical experiments bear out the theoretical predictions and intuitions, with some favorable properties and results compared to the standard frictional optimization. There is much more to explore particularly with larger-scale applications in order to determine the full impact of these novel features.  

An important direction is to optimize among ECD optimizers, with the guidance of \eqref{eq:vol-I-maintext}\eqref{psvolIII}.
Combining ECD with adaptive stepsizes as in e.g. \cite{kingma2017adam}\  or with weight averaging \cite{61aa9e9cc965421e82d7b7042c61abc8} -- distributed predictably as in \eqref{psvolIII}\eqref{eq:vol-I-maintext} --  may be worthwhile.
It is intriguing that the stochasticity of SGD, which introduces hard-to-characterize noise, is largely replaced in BBI by the bounces;  the relativistic speed limit constrains the variance $\langle\dot\bTh^2\rangle$ entering into minibatch-induced noise (\S\ref{app:stochastics}).  Since the randomness induced by bouncings is intrinsic (not tied to the batches), we expect that bouncing ECD algorithms are robust against adversarial attacks based on batch ordering such as \cite{shumailov2021manipulating}.
It will also be interesting to  assess the way that ECD's distinctive dynamics affects representation learning, e.g. adapting to ECD the frameworks developed in \cite{roberts2021principles} and \cite{2022arXiv220303466Y} for initialization and hyperparameter scalings with network size.  

\section*{Acknowledgements}
We are grateful to the reviewers for their stimulating comments and questions.
We would like to thank Dan Roberts for many useful discussions and for sharing with E.S. his work  with Josh Batson and Yoni Kahn connecting inflationary theory and optimization, stimulating this work.  We are very grateful to George Panagopoulos for early collaboration and numerous contributions, along with Thomas Bachlechner for fruitful intermediate collaboration.  We thank many others for useful discussions including Guy Gur-Ari, Miranda Cheng, Ethan Dyer, Alice Gatti, Daniel Kunin, Aitor Lewkowycz, Luisa Lucie-Smith, Hiranya Peiris, Andrew Pontzen,  Uros Seljak, Vasu Shyam, Kendrick Smith, Scott Tremaine, Sho Yaida, Zhiyong Zhang and participants of the 2021 Modern Inflationary Cosmology and String-Data workshops. 
Our research is supported in part by the Simons Foundation Investigator and Modern Inflationary Cosmology programs, and the National Science Foundation under grant number PHY-1720397. 
Some computing was performed on the Sherlock cluster. We thank Stanford University and the Stanford Research Computing Center for 
support and computational resources.


\nocite{langley00}

\bibliography{LogI}
\bibliographystyle{icml2022}

\newpage
\appendix
\onecolumn
\section{Appendix:  Further Details of ECD.}\label{app:theory-details}

Here we collect further details and comments on the theory of ECD and BI.  Section \ref{app:derive-eoms-S-H} spells out the derivation of the BI dynamics from the framework of Lagrangian and Hamiltonian classical mechanics in physics.  In \S\ref{app:other-ECD}\ we give additional examples of ECD dynamics. Section \ref{app:ps-vol-details} supplies details of the calculation of the phase space volume \eqref{eq:vol-I-maintext}.   Section \ref{app:stochastics} explores the sharp distinctions between the stochastic version of BI and SGDM, finding reduced diffusion in BI due to the speed limit.  This enables BBI to explore the landscape with less hard to characterize noise (whose role is replaced by the bounces of BBI).   

\subsection{BI Dynamics From the Action and Hamiltonian}\label{app:derive-eoms-S-H}

Optimization is analogous to a physical particle propagating on the loss landscape.  
In physics, the equations of classical mechanics can be efficiently derived by extremizing an {\it action functional} $S=\int dt L$ in terms of a {\it Lagrangian} $L$.  Their first-order form derives from a {\it Hamiltonian} $H$.  The absence of explicit time-dependence in the $S$ or equivalently in $H$ yields conservation of energy.   This formalism may be found in many classical mechanics textbooks, and is succinctly explained without prerequisites in e.g. Lectures 6-8 of \cite{Susskindminimum} or \cite{landau1976mechanics} Chapters 1,2,7.

We quickly review this formalism in App.~\ref{app:derive-eoms-SGDM}, with an emphasis on obtaining GDM continuum equations from a \emph{time-dependent} action.    There is a relation, known as \emph{Noether's theorem}, between conservation of energy and the absence of explicit time dependence in the specification of the physical theory.  
Conversely, time-dependence in the specification of the theory can capture the effect of the the
dissipation of energy into an additional sector of a larger energy-conserving system.
For a discussion of Noether's theorem,  see e.g.~ (\cite{landau1976mechanics}, Chapter 2).
In this appendix, we will spell out these results in the case relevant for its relation to optimization dynamics.

After explaining this in the non-relativistic case analogous to GDM, we specialize the formalism to BBI's energy-conserving dynamics in App.~\ref{app:derive-eoms-BBI}  with a \emph{time-independent} action and briefly introduce other ECD for optimization in App. ~\ref{app:other-ECD}.
\subsubsection{Actions, Hamiltonians and Gradient Descent with Momentum}\label{app:derive-eoms-SGDM}
If a mechanical system is identified by a vector of position $\bTh$ and its velocity $\dot{\bTh}$, where the dot denotes a time-derivative, its dynamics (meaning the equations governing its time evolution starting from an initial point $\bTh_0$ with initial velocity $\dot\bTh_0$) can be encoded in an \emph{action functional}
\begin{equation}
    S \equiv \int dt L(\bTh,\dot{\bTh}, t)\,,
\end{equation}
constructed in terms of a function $L(\bTh,\dot{\bTh}, t)$ called \emph{Lagrangian}.
 The time evolution of the system is then determined by the Euler-Lagrange equations
\begin{equation}\label{eq:EL}
    \frac{d}{d t} \left( \frac{\partial L}{\partial \dot{\bTh}} \right) =
\frac{\partial L}{\partial \bTh} .
\end{equation}
 which result from demanding that the variation of the action vanishes \cite{Susskindminimum}.

To gain some intuition, let us apply this formalism to the action
\begin{equation}\label{eq:SGDM}
    S_{GDM} = \int dt e^{\frac{f}{m} t} \left[ \frac{m \dot{\bTh}^2}{2} - V (\bTh)
\right]\:.
\end{equation}
The Euler-Lagrange equations \eqref{eq:EL} give
\begin{equation}\label{eq:fric1}
    m \ddot{\bTh} = -\bnabla V -f \dot{\bTh}
\end{equation}
This second-order system describes the motion of a particle of mass $m$ under the force $-\bnabla V$, slowed down by friction controlled by the coefficient $f$.

A convenient and conceptually important rewriting of the same physics in a first-order fashion can be obtained by switching to \emph{Hamiltionian} formalism. To do so, one defines the vector of \emph{momenta}
\begin{equation}\label{eq:pi}
    \bPi \equiv \frac{\partial L}{\partial \dot{\bTh}}
\end{equation}
and an \emph{Hamiltonian} function
\begin{equation}\label{eq:Ham}
H(\bTh,\bPi, t)\equiv  \bPi\dot{\bTh}-L (\bTh,\dot{\bTh}, t)   \;,
\end{equation}
where all the instances of $\dot{\bTh}$ in \eqref{eq:Ham} are substituted  for $\bPi, \bTh$ and $t$ using \eqref{eq:pi}.
In this language, the Euler-Lagrange equations \eqref{eq:EL} are translated to Hamilton-Jacobi equations
\begin{equation}
    \dot{\bTh} = \frac{\partial H}{\partial \bPi}\;,\qquad \dot{\bPi} = -\frac{\partial H}{\partial \bTh}.
\end{equation}

Applying the  Hamiltonian formalism to the GDM  action \eqref{eq:SGDM} transforms the system \eqref{eq:fric1} into
\begin{equation}\label{eq:fr2}
    \dot{\bTh} = \frac{\bPi}{m}\;,\qquad \dot{\bPi} = -\frac{f}{m}\bPi -\bnabla V . 
\end{equation}
Substituting the first equation here into the second recovers the second order form \eqref{eq:fric1}.
The system \eqref{eq:fr2} is a continuum version of the GDM algorithm, as discussed recently e.g. in \cite{RGD2020, muehlebach2021optimization} and references therein. Extra care with the map continuous $\leftrightarrow$ discrete is required when considering the effects of stochasticity added to GDM by noisy gradients (e.g.~\cite{Yaida2018arXiv181000004Y, Kunin-Ganguli-etal-DBLP:journals/corr/abs-2107-09133}  ). We will analyze the effect of noise for BBI in Sec.~\ref{app:stochastics}.

The Hamiltonian is interpreted as the total \emph{energy} $E$ of the system. 
However, in the GDM example described so far, one can easily check that it not conserved, i.e. its value changes during the evolution:
\begin{equation}
    \frac{d}{dt} H_{GDM}  \neq 0 \qquad\qquad E\; \text{ is not conserved}.
\end{equation}
Physically, energy has been lost due the heat generated by friction.
Mathematically, this non-conservation is a consequence of the Lagrangian \eqref{eq:SGDM} depending \emph{explicitly} on $t$.

\subsubsection{Energy Conserving Dynamics and BBI}\label{app:derive-eoms-BBI}

The Born-Infeld (BI) dynamics \eqref{eom}-\eqref{eomdiscrete2} can be derived at the continuum level from the action
\begin{equation}\label{DBIS}
S_{BI}=\int dt L= -\int dt V(\bTh)\sqrt{1-\frac{\dot{\bTh}^2}{V(\bTh)}} .
\end{equation}
Using \eqref{eq:pi} we derive the momentum
\begin{equation}\label{eq:pitheta}
\bPi =\frac{\partial L}{\partial \dot{\bTh}}=\frac{\dot\bTh}{\sqrt{1-\frac{\dot{\bTh}^2}{V(\bTh)}}}
\end{equation}                           
and from \eqref{eq:Ham} the Hamiltonian 
\begin{equation}\label{Ham}
H_{BI} = \bPi\dot\bTh - L=\frac{V(\bTh)}{\sqrt{1-\frac{\dot{\bTh}^2}{V(\bTh)}}}=\sqrt{V(\bTh)(V(\bTh)+\bPi^2)}.
\end{equation}
From this we extract the continuum equations of motion \eqref{eom}, which we repeat here for convenience of the reader:
\begin{equation}\label{eq:eomApp}
    \dot\bTh = \bPi \frac{V(\bTh)}{E} , \qquad
\dot\bPi =-\frac{\bnabla V}{2}\left(\frac{E}{V(\bTh)}+\frac{V(\bTh)}{E}\right),
\end{equation}
where we defined $E \equiv H_{BI}$.
As explained in the main text, a first order symplectic integration of these equations produces the discretized BI evolution.

For the BI dynamics, the Hamiltonian \eqref{Ham} is conserved, meaning that its value is fixed during the evolution generated by \eqref{eq:eomApp}:
\begin{equation}
    \frac{d}{dt} H_{BI}  = 0 \qquad\qquad E \text{ is a constant}.
\end{equation}
From eq. \eqref{Ham} we see  that the speed of motion in parameter space $|\dot\bTh|$ cannot exceed $V$, and conversely that $\dot\bTh^2$ will not vanish for $0< V< E$.   This is very different from gradient descent with or without momentum.  

\subsection{Other Examples of ECD}\label{app:other-ECD}

There is a large space of physical models with energy conserving, frictionless dynamics that slows to a stop as $V\to 0$.  Another mentioned in \S\ref{sec:intro} has action (in the language of \S\ref{app:derive-eoms-S-H}) 
\begin{equation}
    S=\int dt \left(\frac{1}{2}m(\bTh)\dot\bTh^2  \right) = \int dt \left(\frac{1}{2 V(\bTh)}\dot\bTh^2 \right)
\end{equation}
Here the objective $V\to 0$ occurs when the mass ($= 1/V$) blows up, causing the particle to slow to a stop.

For such a model the first order equations of motion read
\begin{equation}\label{eq:RL}
    \dot{\bTh} = \bPi V(\bTh)\qquad\dot{\bPi} = -\frac{E}{V(\bTh)}\bnabla V \:,
\end{equation}
descending from the Hamiltonian 
\begin{equation}
    H = \frac{1}{2} V(\bTh) \bPi^2 .
\end{equation}
A symplectic integration of the equations \eqref{eq:RL} produces another example of an optimization algorithm that conserves energy and shares the properties of BBI described in the main text.

Another example analogous to BI and also inspired by a similar theoretical cosmology model \cite{Mathis:2020vdn}\ has a logarithmic rather than a square root branch cut enforcing the speed limit:
\begin{equation}
    S = \int dt \left[ -\frac{V(\bTh)}{2}\log\left(1-\frac{\dot{\bTh}^2}{V(\bTh)}\right) \right]
\end{equation}
In all these cases, including BI, the shifted loss function $V(\bTh)$ \eqref{eq:SL-V-F-shift-def} may be replaced by any function of it that approaches 0 as $V\to 0$.  In general, there is a high-dimensional space of possible ECD models, all with the same general properties listed in Table \ref{table:comparison}.

In the context of theoretical cosmology, the classes of models of early universe inflation (accelerated expansion of the universe) which motivated ECD are distinctive in terms of their dynamics and their observational signatures.  These enable empirical discrimination between them and the GD-like dynamics of so-called slow-roll inflation. \cite{ Armendariz-Picon:1999hyi, Alishahiha:2004eh, Chen:2006nt, Cheung:2007st}.  In the present work, we similarly see sharp distinctions between frictional and ECD optimization which will be interesting to further explore and exploit.     

\subsection{Details in the Calculation of the Phase Space Volume Formula Predicting the Distribution of Results}\label{app:ps-vol-details} 
Here we supply the steps deriving \eqref{eq:vol-I-maintext} from \eqref{psvol}.  

The total volume is given by
\begin{equation}\label{psvol}
\text{Vol}({\cal M}) = \int d^n \pi d^n\theta \delta(\sqrt{V(V+\bPi^2)}-E)
\end{equation}  
Writing $|\bPi|=\tilde\pi$ and doing the angular integral in the momenta gives
\begin{equation}\label{psvolnext}
\text{Vol}({\cal M}) = \frac{2\pi^{n/2}}{\Gamma(n/2)}\int d^n\theta \int_0^\infty d\tilde\pi \tilde\pi^{n-1}\delta(\sqrt{V(V+\tilde\pi^2)}-E)\;.
\end{equation}
Using the $\delta$ function to do the $\tilde\pi$ integral then yields
\begin{equation}\label{psvolIII}
\text{Vol}({\cal M}) = \frac{2\pi^{n/2}}{\Gamma(n/2)}\int d^n\theta \frac{E}{V}\left(\frac{E^2}{V}-V \right)^{\frac{n-2}{2}} \;.
\end{equation}

 The analogous formula for pure momentum (without speed limit or friction) would be 
 \begin{equation}\label{eq:volMom}
     \text{Vol}({\cal M})\propto \int d^n\theta (E-V)^{\frac{n-2}{2}} ~~~~~ \text{frictionless~non-relativistic~momentum}
 \end{equation} and does not exhibit the dominance at low $V$ of BBI \eqref{psvolIII}.    Different forms of ECD will produces analogous formulas to \eqref{psvolIII}, sometimes enhancing the dominance at small $V$.  It will be interesting to explore the performance of such generalizations, including the replacement of $V$ with a more general function $g(V)$ of the loss. 

Let us consider the contributions to the integral near the global minimum (or degenerate minima) at some $\bTh_0$ where $V\to 0$.  We can expand the measure, 
\begin{equation}
d^n\theta = |\bTh-\bTh_0|^{n-1}d|\bTh-\bTh_0|d\Omega_{\bTh-\bTh_0}
\end{equation}
and also expand $V\propto |\bTh-\bTh_0|^\kappa$ for some $\kappa$.   From (\ref{psvolIII}) we see that the singularity there is integrable only for $\kappa < 2$.  One would expect $\kappa=2$ for a smooth potential $V$ with global minimum $V_\text{global}=0$.  This leads to a logarithmically divergent volume as we approach $V_\text{global}$.

More generally, we can compare the phase space volume in the basins of attraction of different local minima.  
Starting from (\ref{psvolIII}) and working in the regime of $V\ll E$, we can write the measure near the $I$th local (or global) minimum as 
 
\begin{equation}\label{psvol-loc-min-I}
\text{Vol}({\cal M}_I) = \frac{2\pi^{n/2}}{\Gamma(n/2)}E^{n-1}\int d^n(\theta-\theta_I) V^{-n/2}
\end{equation}
Near a minimum, $V$ is quadratic, giving (after an orthogonal diagonalization of its Hessian)
\begin{equation}\label{Vexp}
V\simeq V_I + \frac{1}{2}\sum_{i=1}^n m_{Ii}^2 (\theta_i-\theta_{Ii})^2\;.
\end{equation}
Defining
\begin{equation}\label{orthog-eta-coords}
\eta_{iI}=m_{iI}(\theta_i-\theta_{Ii}) \Rightarrow V\simeq V_I + \frac{1}{2}\sum_{i=1}^n \eta_{Ii}^2\;,
\end{equation}
we get 
\begin{equation}\label{psvol-loc-min-II}
\text{Vol}({\cal M}_I) = \frac{2\pi^{n/2}}{\Gamma(n/2)}\frac{E^{n-1}}{\prod_i m_{Ii}}\int d^n\eta V^{-n/2}=\left(\frac{2\pi^{n/2}}{\Gamma(n/2)}\right)^2\frac{E^{n-1}}{\prod_i m_{Ii}}\int d\eta\frac{ \eta^{n-1}}{ (V_I+\frac{1}{2}\eta^2)^{n/2}}\;,
\end{equation}
where $\eta \equiv |\eta_I|$ exhibiting a logarithmic divergence as $V_I\to 0$ as in \eqref{eq:vol-I-maintext}.  

Integrating $\eta$ from 0 to 1 here, as an estimate of the contribution to the measure of this $I$th minimum, gives a formula for the volume in terms of a hypergeometric function:
\begin{equation}\label{psvol-dointegral}
\text{Vol}({\cal M}_I) = \left(\frac{2\pi^{n/2}}{\Gamma(n/2)}\right)^2\frac{E^{n-1}}{\prod_i m_{Ii}}\frac{V_I^{-n/2}}{n}  ~_2F_1(\frac{n}{2}, \frac{n}{2}, \frac{n}{2}+1, -\frac{1}{2 V_I}) \;.
\end{equation}

\subsection{Further Study of BI in the Stochastic Case}\label{app:stochastics}

As noted in the main text, it is often necessary to separate the data (or the input points sampled in solving differential equations) into minibatches $x_B$, switching from batch $B$ to batch $B+1$ at time $t_B$, 
i.e. switching to a new minibatch every $\Delta (t_{B+1}-t_B)/\Delta t$ steps.  As explained in \S\ref{sec:minibatches}, the essential features of ECD persist in this case, as is borne out by the experiments.

In this section we explore in a little more detail the stochastic behavior of BI.  
Similarly to the noise-free case, we find clear distinctions between this class of optimizers and friction-based ones such as SGDM.  It would be interesting to extend these preliminary studies in the future.  We note interesting prior work \cite{Tolley_2010}\ on stochastic effects in the DBI theory of early universe inflation \cite{Alishahiha:2004eh}.

The transitions to new batches means that at a given time, the algorithm evolves in the loss landscape given by $V^B=V(\bTh; x_B)$ rather than the full loss $V(\bTh;x)$. 
The loss therefore behaves like a time dependent potential in the corresponding Hamiltonian system. In this physical system, that would generally lead to energy non-conservation by an amount determined by the strength of the time dependence, $(V^{B+1}-V^B)/V^B\equiv \Delta V_B/V^B$.  
In our algorithm, however, the situation is somewhat different since we keep the energy fixed via the rescalings described in the main text.  

To begin let us describe the rescaling required to preserve energy $E$ in the presence of minibatches (even in the absence of discretization error). From \eqref{Ham} we have $\bPi^2=\frac{E^2}{V}-V$.  Thus as we transition from batch $B$ to batch $B+1$, we have to rescale by a factor
\begin{equation}\label{eq:lambda-rescaling}
    \lambda = \frac{\pi_{i, B+1}}{\pi_{i, B}} = \sqrt{\frac{E^2/V^{B+1}-V^{B+1}}{E^2/V^B-V^B}}
\end{equation}
in order to conserve energy.

Before commenting on the small learning rate regime, we first note that the prescription of \cite{Yaida2018arXiv181000004Y}, deriving averaged correlations of observables under the assumption of a late time steady state distribution, extends to the BI case.  Following \cite{Yaida2018arXiv181000004Y}, if we assume an equilibrium distribution at late times, we can compute appropriate correlation functions in the putative distribution.  Those arising at the quadratic order in $\bTh, \bPi$ are: 
\begin{equation}
\langle [| \frac{\partial_iV^B}{V^B}\theta_j |]\rangle = \langle [| \frac{\partial_jV^B}{V^B}\theta_i |], ~~
\langle ~ V( \theta_i\pi_j + \theta_j\pi_i)~ \rangle = \Delta t E ~\langle ~  \frac{1}{4}(\partial_iV\theta_j + \partial_j V \theta_i)-[| \frac{{(V^{B })^2}}{2 E^2}\pi_i\pi_j |] ~\rangle
\end{equation}
where as in that work, $V^B$  denotes the loss function for batch $B$, $[|\dots |]$ represents a minibatch average, whereas $\langle \dots \rangle$ represents the expectation value in the steady state distribution.  Here we assumed $V^B\ll E$ to simplify the formulas.   The last relation here is comparable to equation (28) in \cite{Yaida2018arXiv181000004Y}, with the similarity clearer upon noting from \eqref{eq:pitheta}-\eqref{Ham} that the velocity is given  by $v_i=\dot\theta_i =\pi_i V/E $.  Although the relations look similar, they encode very different behavior:  Whereas in SGDM $\langle v^2\rangle$ is generally nonzero (and in ordinary Brownian motion it is approximately constant), in BI this quantity is bounded by $V$ because of the loss-dependent speed limit. 

Next, we briefly study the small $\Delta t$ regime.  By plugging the second equation in \eqref{eom} into the first
and taking into account the additional time-dependence in $V$ and in $\pi$ from the rescaling \eqref{eq:lambda-rescaling}, we obtain the second-order equations
\begin{equation}\label{eq:continuum-second-order}
    \ddot\theta_i =-\frac{1}{2}\partial_i V + \dot\theta_i (\frac{\dot\bTh\cdot\nabla V}{V})
 +  \dot\theta_i\sum_B \delta(t-t_B)\left(1-\lambda(\Delta V^B)+\frac{\Delta V^B}{V^B}\right).
\end{equation}
For the present discussion we do not include the billiards-inspired bounces prescribed in the algorithm.
All of the terms multiplying $\delta(t-t_B)$ are of order $\Delta V^B/V^B$ and may be expected to ensemble-average to zero (including when multiplied by other quantities such as $\dot\theta_i$ that are not tied to the random choice of new batch).   We could have obtained the resulting equation equally well from our discretization of BI by plugging \eqref{eomdiscrete2}\ into \eqref{eomdiscrete1}.   

We may contrast this to a similar small learning-rate limit \cite{Kunin-Ganguli-etal-DBLP:journals/corr/abs-2107-09133}
for SGDM:
\begin{equation}\label{eq:SGDM-SDE}
    \frac{\Delta t}{2}(1+\beta)\ddot\bTh +(1-\beta)\dot\bTh =-\nabla V^B
\end{equation}
derived from the update rule
\begin{equation}
    {\bf v}_{k+1}-\beta {\bf v}_k = -\nabla V_k, ~~~~ \bTh_{k+1}-\bTh_k =\Delta t~ {\bf v}_{k+1}\;.
\end{equation}
We note the absence of a $\Delta t$ in the first update rule, which is standard in machine learning but different from the standard discretization of the analogous non-relativistic classical particle model and hence different from the non-relativistic regime of our discretized BI model.

Aside from the distinction in the placement of $\Delta t$ factors, we note the presence of a friction term in \eqref{eq:SGDM-SDE}\ and its absence in \eqref{eq:continuum-second-order}.  In both cases, the batch dependence in $V(\bTh, t)$ constitutes a source of noise, in general with nontrivial statistics.  

For now, we contrast the form that Brownian motion takes in the two systems in this small-learning-rate regime.  According to \cite{Kunin-Ganguli-etal-DBLP:journals/corr/abs-2107-09133}, SGDM exhibits anomalous Brownian motion, with the variance of the motion behaving as $\langle\bTh^2\rangle \propto t^\kappa$ for some positive exponent $\kappa$.  The minimal version of Brownian motion in physics can be derived from the Langevin equation $\ddot\theta +\gamma \dot\theta =\xi$ \footnote{See e.g. \href{https://web.stanford.edu/~peastman/statmech/friction.html}{https://web.stanford.edu/~peastman/statmech/friction.html} .} where the noise term satisfies $\langle \xi(t) \xi(t')\rangle=\delta (t-t')$, which is very loosely similar to \eqref{eq:SGDM-SDE}, leading to the relation $\frac{d}{dt}\langle \theta^2 \rangle = 2 \langle \dot\theta^2 \rangle/\gamma  = \text{constant} $.    

In comparing to \eqref{eq:continuum-second-order}, although there is no friction term, we can see that for a potential basin of the form $\frac{1}{2}m^2\bTh^2$, for speed-limited motion along the gradient direction we have $\langle \dot\bTh\cdot\nabla V/V\rangle \sim -1 + \text{fluctuations} $.  Comparing to the standard Brownian motion result just summarized, we again expect $\frac{d}{dt}\langle \bTh^2 \rangle \sim 2 \langle \dot\bTh^2\rangle $.  But in contrast to that case and SGDM, for BI the speed limit strongly constrains the right hand side.  Thus we expect much less diffusion for BI.  Instead, for BI it is the billiards-inspired bounces lead to mixing over the phase space.  Clearly it would be interesting to flesh out these distinctions in more detail, along with their implications for representation learning in ML and for coverage of PDE solutions.     

\section{Details on the PDE Problems} \label{app: PDE}
Here we describe in some details the class of PDE problems on which we tested the BBI algorithm.

We considered the class of nonlinear Poisson Dirichlet problems defined by
\begin{equation}\label{eq:PDE-prob}
    \textrm{find } u \textrm{ such that}\qquad\left\{ \begin{array}{lcl}
         \Delta u+ u^2 = f& & x\in \Omega \\
         u=f_0& &x\in \partial\Omega 
    \end{array}\right.\;,
\end{equation}
where $\Omega$ is a domain in $\mathbb{R}^d$ with boundary $\partial\Omega$, $\Delta = -\sum_{i=1}^d \partial_{x_i}^2$, and $f$ is a known function.

To design a specific problem we then proceeded backwards: we chose a certain $u$ with a non-trivial shape, plugged it on the left hand sides of (\ref{eq:PDE-prob}) and computed $f$ (and $f_0$). This allowed us to know analytically one of the solutions of the problem.
Specifically, for the experiment presented in Sec.~\ref{sec:PDEs} we worked in the unit-ball in $\mathbb{R}^2$ and considered as analytic solution the function \begin{equation}\label{eq:analytic_sol_pde}
    u_\textrm{an.}  =  \sin^2(20(x_1^2+x_2^2))
\end{equation} 
of which we plot a section in Fig.~\ref{fig:sols-PDE} (left).
\begin{figure*}[ht]
    \centering
    \hfill
    \includegraphics[width=0.3\linewidth]{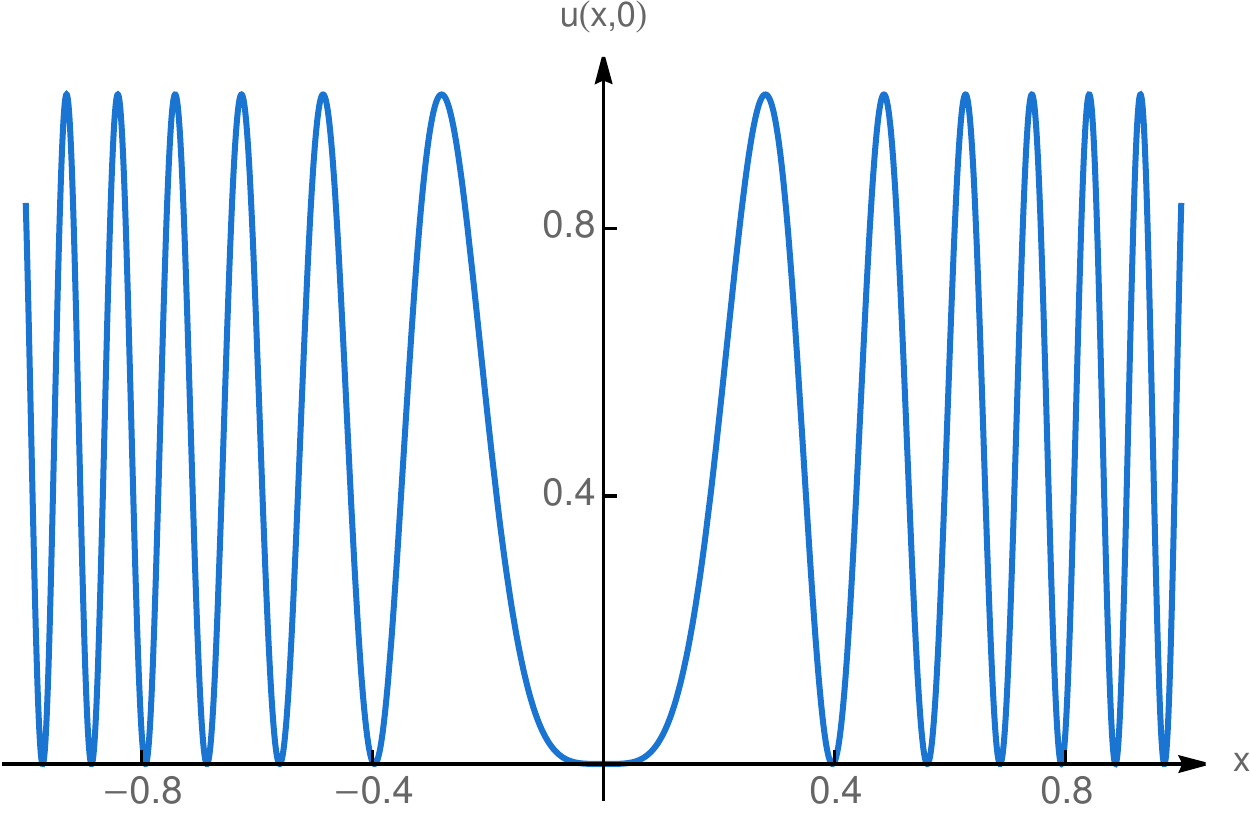}
    \hfill
    \includegraphics[width=0.3\linewidth]{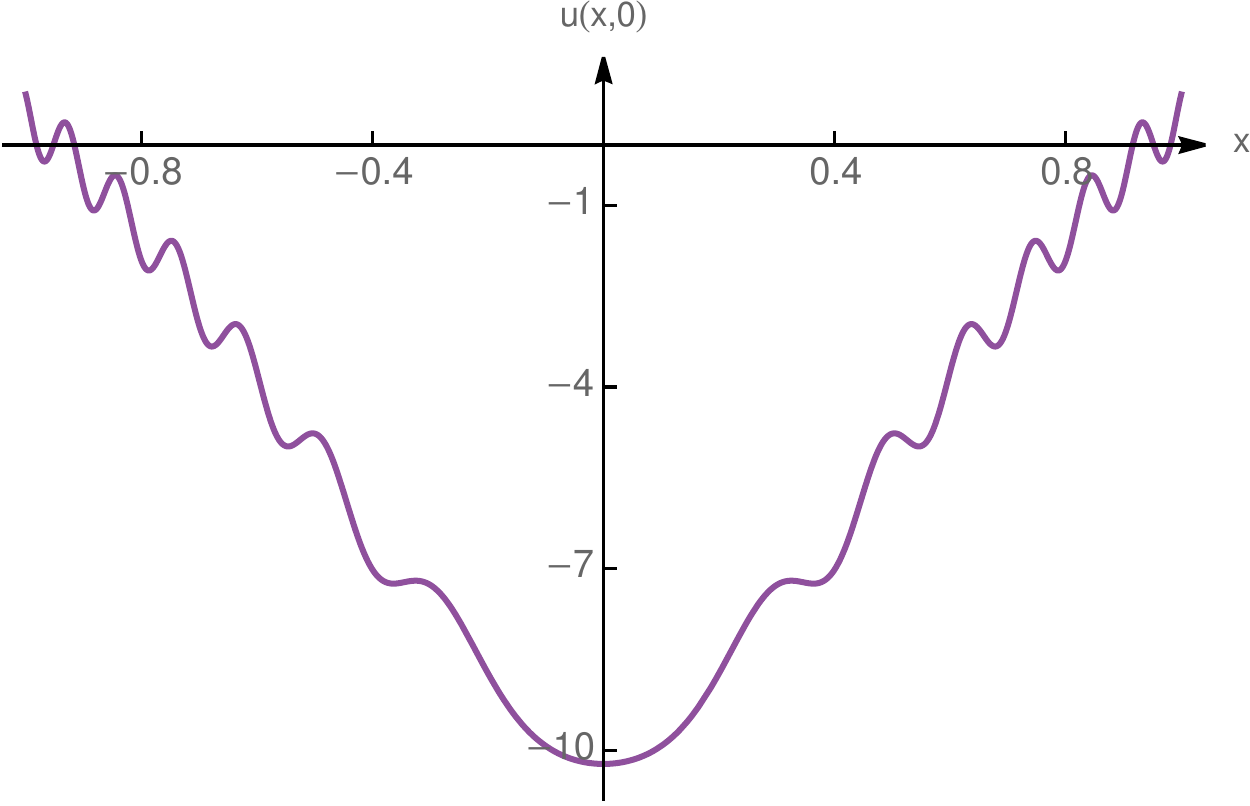}
    \hfill
    \caption{The solutions to the PDE problem studied in Sec.~\ref{sec:PDEs}. On the left a section at $x_2 = 0$ of the analytic solution \eqref{eq:analytic_sol_pde}; on the right, a section of the other solution determined numerically with a shooting method.}
    \label{fig:sols-PDE}
    \vskip -0.2in
\end{figure*}
    
We have chosen such a wiggly solution in order to obtain a relatively hard numerical problem. Plugging it in Prob. (\ref{eq:PDE-prob}) , we extract $f$ and $f_0$:
\begin{equation}\label{eq:fprob2}
    f = \frac{1}{8}\left(3-4(1+6400(x_1^2+x_2^2))\cos(40(x_1^2+x_2^2)) + \cos(80(x_1^2+x_2^2))-640\sin(40(x_1^2+x_2^2)) \right) \,
\end{equation}
and $f_0 = 0.833469$. 
The problem \eqref{eq:analytic_sol_pde} being nonlinear, we can ask whether other solutions exist. As can be shown by employing radial symmetry together with a shooting method, the answer is affirmative: two solutions exist with second one represented in Fig.~\ref{fig:sols-PDE} (right).
We stress that spherical symmetry and the reduction to a one-dimensional radial equation are used only to employ a standard method to assess the existence of other solutions, but has not been employed in any way in the solution of Prob.~(\ref{eq:PDE-prob}) with NNs, since ultimately the goal is to being able to solve higher-dimensional problems with no symmetry assumptions.

For the experiments in Sec.~\ref{sec:PDEs}, we then used as an ansatz for solving \eqref{eq:PDE-prob} a residual network defined as
\begin{equation}\label{eq:net1}
    \begin{array}{lllr}
        y_1^i& = \sigma(\sum_{\mu = 1}^{n}W_0^{i \mu} x_\mu +b_0^i)& &  \\
        y_{s+1}^i&=  y_{s}^i+\sigma( \sum_{j = 1}^{N} W_{s}^{i  j}  y_{s}^j +b_s^i)  & &\\
        y_{D+2}&=   \sum_{j = 1}^{N} W_{D+1}^{ j}  y_{D+1}^j +b_{D+1}  &&
    \end{array}
\end{equation}
where $s = 1,\dots D$ and $i = 1,\dots, N$ with $D$ and $N$ being respectively the depth and width. $d$  is the dimension of the PDE problem, and in our examples restrict to $d = 2$. As \emph{activation function} $\sigma$ we chose the logistic sigmoid.   We briefly experimented with other smooth activation functions (tanh, sin) and not observing degradation of performance, we focused on the sigmoid to show the distribution of results.
The sets of weight and biases, $W$ and $b$, define the set of parameters collectively denoted as $\bTh$ in the main text.
Summing up, the sequence of transformations in \eqref{eq:net1} is defining a function from $\mathbb{R}^d \to \mathbb{R}$, depending on these parameters, which is used as an ansatz for the PDE solution. Schematically, 
\begin{equation}
        \mathcal{N}(x;\bTh) \equiv y_{D+2}(y_{D+1} (\dots)) .
\end{equation}

In our experiments we used 3 middle layers of width 200, i.e. $D=3$ and $N=200$. The loss function \eqref{eq:lossPDE} is then constructed with an $L^2$ regularization term R$(\bTh) = 2\times 10^{-4} \bTh^2$, $\gamma = 10^{5}$ and by sampling respectively $10^{4}$ and $10^{3}$ random points from the unit ball in $\mathbb{R}^2$ and its boundary.

\end{document}